\documentclass{article}

\usepackage{microtype}
\usepackage{graphicx}
\usepackage{booktabs} 

\usepackage{hyperref}

\usepackage[accepted]{icml2023_arxiv}

\usepackage{amsmath}
\usepackage{amssymb}
\usepackage{mathtools}
\usepackage{amsthm}

\theoremstyle{plain}
\newtheorem{theorem}{Theorem}
\newtheorem{proposition}{Proposition}
\newtheorem{lemma}{Lemma}

\theoremstyle{definition}

\newtheorem{assumption}{Assumption}
\theoremstyle{remark}
\newtheorem{remark}[theorem]{Remark}
\newtheorem{example}{Example}

\usepackage[textsize=tiny]{todonotes}

\usepackage{microtype}
\usepackage{bbm}
\usepackage{graphicx}
\usepackage{subcaption}
\usepackage{booktabs} 
\usepackage{makecell}
\usepackage{wrapfig}
\usepackage{diagbox}
\usepackage{enumitem}%
\usepackage{amsfonts}
\usepackage{xcolor}

\usepackage{cprotect}
\usepackage{multirow}
\usepackage{fancyvrb}

\newcommand{\boldone}{\mathbbm{1}}

\newcommand{\rev}[1]{{\color{black}#1}}
\newcommand{\revold}[1]{{\color{black}#1}}

\newcommand{\EnbPI}{\Verb|EnbPI|}
\newcommand{\SPCI}{\texttt{SPCI}}

\newcommand{\betaReal}{\hat{\beta}}

\newcommand{\bP}{\mathbb{P}}
\newcommand{\indep}{\perp \!\!\! \perp}

\newcommand{\R}{\mathbb{R}}
\newcommand{\Ctalpha}{\widehat{C}_{t-1}(X_t)}
\newcommand{\hatQt}{\widehat{Q}_t}
\usepackage{wrapfig}
\newcommand{\hatQtz}[1]{\widehat{Q}_t(#1)}

\newcolumntype{P}[1]{>{\centering\arraybackslash}p{#1}}

\icmltitlerunning{Sequential Predictive Conformal Inference for Time Series}

\begin{document}

\twocolumn[
\icmltitle{Sequential Predictive Conformal Inference for Time Series}



\icmlsetsymbol{equal}{*}

\begin{icmlauthorlist}
\icmlauthor{Chen Xu}{yyy}
\icmlauthor{Yao Xie}{yyy}
\end{icmlauthorlist}

\icmlaffiliation{yyy}{H. Milton Stewart School of Industrial and Systems Engineering,
Georgia Institute of Technology, Atlanta, Georgia, USA}

\icmlcorrespondingauthor{Yao Xie}{yao.xie@isye.gatech.edu}

\icmlkeywords{Machine Learning, ICML}

\vskip 0.3in
]



\printAffiliationsAndNotice{}  

\begin{abstract}
We present a new distribution-free conformal prediction algorithm for sequential data (e.g., time series), called the \textit{sequential predictive conformal inference} (\texttt{SPCI}). We specifically account for the nature that time series data are non-exchangeable, and thus many existing conformal prediction algorithms are not applicable. The main idea is to adaptively re-estimate the conditional quantile of non-conformity scores (e.g., prediction residuals), upon exploiting the temporal dependence among them. More precisely, we cast the problem of conformal prediction interval as predicting the quantile of a future residual, given a user-specified point prediction algorithm. Theoretically, we establish asymptotic valid conditional coverage upon extending consistency analyses in quantile regression. Using simulation and real-data experiments, we demonstrate a significant reduction in interval width of \texttt{SPCI} compared to other existing methods under the desired empirical coverage.
\end{abstract}

\section{Introduction} 

Uncertainty quantification for prediction algorithms is essential for statistical and machine learning models. 
Sequential prediction or time-series prediction aims to predict the subsequent outcome based on past observations. Uncertainty quantification in the form of prediction intervals is of particular interest for high-stake domains such as finance, energy systems, healthcare, and so on \citep{Harries1999SPLICE2CE,diaz2012review,renew_data_challenge}. Classic approaches for prediction interval are typically based on strong parametric assumptions of time-series models such as autoregressive and moving average (ARMA) models \citep{tsmethod}, which impose strong distribution assumptions on the data-generating process. There need to be principled ways to perform uncertainty quantification for complex prediction models such as random forests \citep{RF} and neural networks \citep{DeepReg}.

Conformal prediction (CP) has become a popular distribution-free technique to perform uncertainty quantification for complex machine learning algorithms. However, conformal prediction for time series has been a challenging case because such data do not satisfy the exchangeability assumption in conformal inference, and thus we need to adjust existing or even develop new \textit{sequential CP} algorithms with theoretical guarantees. The challenges also arise in real-world applications where time series data tend to have significant stochastic variations and strong correlations. These challenges are illustrated via a real-data example for solar energy prediction, as shown in Figure \ref{serial_depend}, where the prediction residuals (using random forest as prediction algorithm) are still highly correlated. Besides the temporal correlation in the prediction residuals (or conformity scores in general), we observe that a notable feature of sequential conformal prediction is that the prediction residuals can be obtained as ``feedback'' to the algorithm. For instance, for one-step ahead prediction, the prediction accuracy of the prediction algorithm is revealed immediately after one-time step. Thus, the recent prediction residuals reveal whether or not the predictive algorithm is performing well for that segment of data. Such feedback structure is illustrated in Figure \ref{compare_resid}, which highlights the conceptual difference between traditional conformal and sequential conformal prediction methods. We specifically exploited such feedback structure in designing the sequential conformal prediction algorithms. 

More precisely, both the traditional conformal inference and the sequential conformal inference considered in this paper are general-purpose wrappers that can be used around any predictive model for any data and proceed by defining ``non-conformity scores''.  However, there are also significant differences: Traditional conformal prediction assumes exchangeable training and test data to obtain performance guarantees, which leads to exchangeable non-conformity scores, and cannot receive feedback during prediction. In contrast, sequential CP observes non-exchangeable data sequences and leverages feedback during prediction.

In this work, we propose a \textit{sequential predictive conformal inference} (\texttt{SPCI}) framework for time series with scalar outputs. The idea is to utilize the feedback structure of prediction residuals in the sequential prediction problem to obtain desired coverage. We specifically exploit the serial dependence across prediction residuals (conformity score) by performing quantile regression using past residuals for the future prediction intervals.; thus, the most recent past residuals contain information about the immediate future ones. Similar to most existing conformal prediction literature, we make no assumptions about the data-generating process or the quality of estimation by the point estimator. Our main contributions are

$\bullet$ The main novelty of \SPCI{} is the \rev{time-adaptive re-estimation of residual quantiles over time, upon leveraging the temporal dependency among residuals.} We use Random Forest for quantile regression here, but \SPCI{} is applicable to other quantile regression methods.

$\bullet$ Theoretically, we obtain asymptotic conditional coverage of the constructed intervals for dependent data, based on prior results for random forest quantile regression. When data are exchangeable, we show that \SPCI{} enjoys the same finite-sample and distribution-free marginal coverage guarantee as traditional conformal prediction methods. 

$\bullet$ Experimentally, we demonstrate competitive and/or improved empirical performance against baseline CP methods on sequential data. In particular, \SPCI{} can obtain significantly narrower intervals on real data without coverage loss. We further demonstrate the benefit of \SPCI{} in multi-step predictive inference.


\subsection{Literature review}

Conformal prediction has been an increasingly popular framework for distribution-free uncertainty quantification. Initially proposed in \citep{conformaltutorial}, CP methods generally proceed as follows. First, one designs a type of ``non-conformity score'' based on the given point estimator $\hat{f}$, where the score measures how different a potential value of the response variable $Y$ is to existing observations. A common choice for such scores in regression problems is the prediction residual. Second, one computes these scores on a \textit{hold-out} set not used to train the estimator $\hat{f}$. Third, the prediction interval is defined as all potential values of $Y$ whose non-conformity score is less than $1-\alpha$ fraction of these scores over the hold-out set. Many existing works such as \citep{inductCP,QOOB,MJ_classification,Candes_classification} utilize this idea for uncertainty quantification in regression or classification problems. Comprehensive surveys and tutorials can be found in \citep{conformalreview,angelopoulos2021gentle}. CP framework are distribution-free and model-free: they require neither distributional assumptions on data nor special classes of prediction functions, hence being particularly attractive in practice. Nevertheless, the desired performance guarantee of CP methods relies on \textit{exchangeability} (e.g., the simplest case is when data are i.i.d.), which hardly holds for time series.

\begin{figure}[!t]
  \centering
  \includegraphics[width=0.75\linewidth]{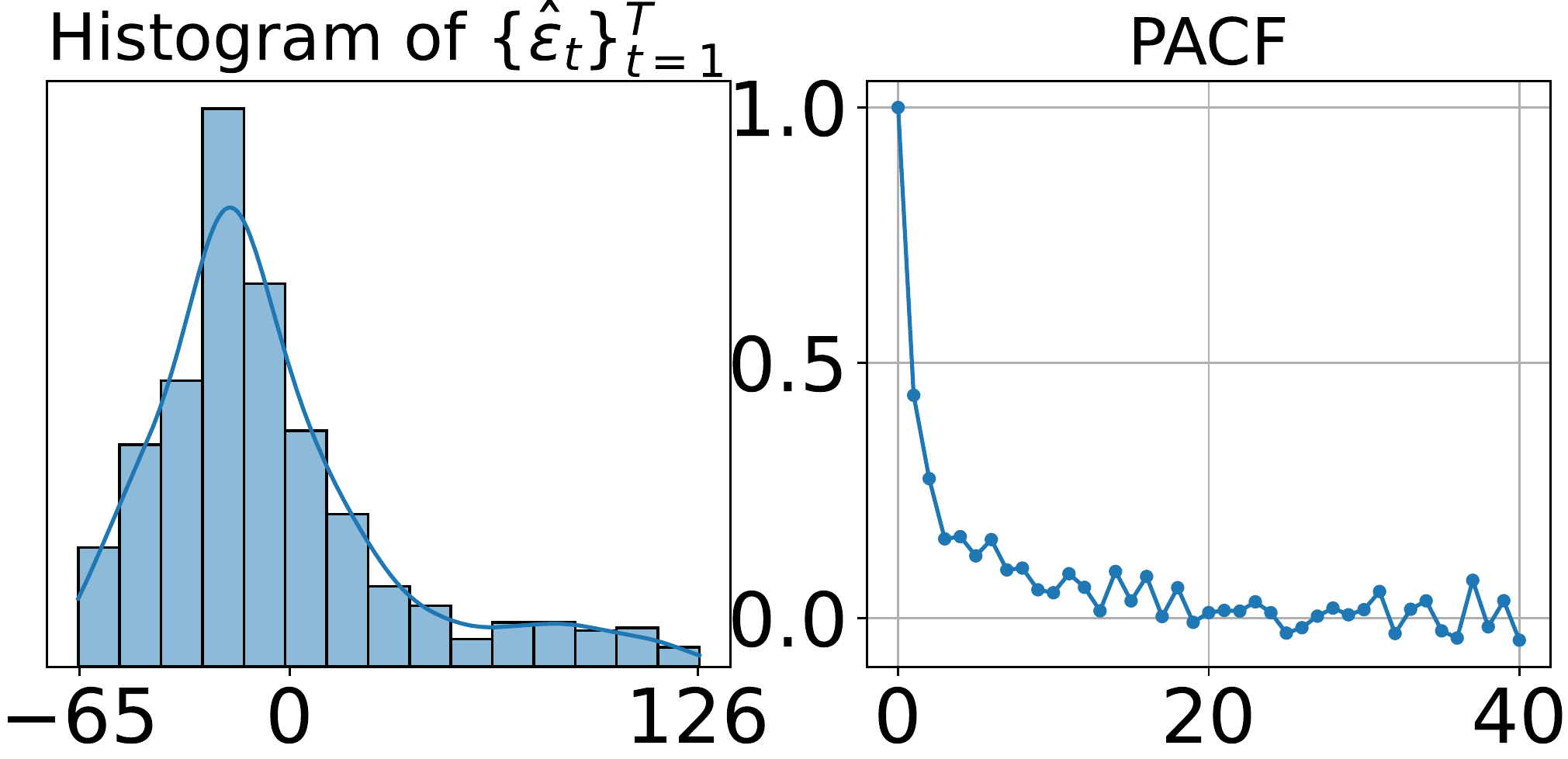}
    \caption{Solar power radiation prediction for downtown Atlanta, Georgia, USA (further explanation in Section \ref{sec_real_data}). We use random forest for one-step-ahead prediction. The histogram of prediction residuals (left) shows that residual distribution is highly skewed, and the partial auto-correlation between residuals (right) shows a significant serial correlation among residuals. Thus, it is essential to consider serial dependency when constructing prediction intervals: the serial dependence means that the most recent past residuals contain information about the immediate future ones.
    }
    \label{serial_depend}
    \vspace{-0.1in}
\end{figure}

Recently, significant efforts have been made to extend CP methods beyond exchangeable data; several are towards building sequential conformal prediction methods. They typically do so via updating non-conformity scores (e.g., prediction residuals) \citep{xu2021AD,xu2021conformal} and/or adjust significance level $\alpha$ based on rolling coverage of $Y_t$. 
This include \citep{Gibbs2021AdaptiveCI,Zaffran2022AdaptiveCP,feldman2022conformalized, lin2022conformal} and specifically, the AdaptCI algorithm, which adjusts the significance level $\alpha$ based on real-time coverage status during prediction---the significance level is lower when the prediction interval at time $t$ fails to contain the actual observation $Y_t$. The prediction intervals thus have adaptive width based on the updated significance levels and maintain coverage on stock market data in practice. 
Furthermore, \cite{RinaNonExchange} proves the coverage gap for non-exchangeable data based on the total variation (TV) distance between the non-conformity scores. The work then proposes NEX-CP, a general re-weighting scheme for non-exchangeable data, where the weights should ideally be chosen to be inversely proportional to the TV distances. The authors demonstrate the robustness of NEX-CP on datasets with change points and/or distribution shifts.
For sequential data, \cite{xu2021conformal} proposes \EnbPI{}, which updates residuals of ensemble predictors during prediction to more accurately calibrate prediction intervals. In practice, \EnbPI{} can maintain desired $1-\alpha$ coverage for different types of time series. 
Despite the existing efforts, these sequential CP methods have not exploited serial correlation among non-conformity scores (cf. Figure \ref{serial_depend})---they only use empirical quantiles (possibly with fixed weights) of past residuals to compute intervals, which is a drastic difference from \SPCI{}. 

\rev{Besides conformal prediction, probabilistic forecasting approaches have also been widely used when building predicting intervals. These approaches typically train a single model to minimize the pinball loss, including the MQ-CNN \citep{wen2017mqcnn}, DeepAR \citep{salinas2020deepar}, Temporal Fusion Transformer (TFT) \citep{lim2021temporal}, etc. However, comparing to \SPCI{} and related CP works, these approaches have two major limitations. First, they are not ``model-free'': special designs of the predictive model and hyper-parameter tuning are required for satisfactory performances. Second, they are not ``distribution-free'': distributional assumptions on time-series are often imposed, such as Gaussianity \citep{salinas2020deepar}. Corresponding theoretical guarantees on constructed prediction intervals are also often lacking. In our experiments, we demonstrate the improved performance of \SPCI{} against DeepAR and TFT.}

\subsection{Connection with related works}

Through theoretical analysis, we find that when using random forest quantile regression, \SPCI{}  can be viewed as adaptively learning the (data-dependent) weights of the prediction residuals/non-conformity scores when constructing the prediction intervals using weighted quantile values. Hence, it has an interesting connection to the recent work \citep{RinaNonExchange}, which develops a general conformal prediction framework for non-exchangeable data. In that work, weights are pre-determined and non-adaptive (such as geometrically decaying weights), and the authors also pointed out that ``how to choose weights optimally ... is an interesting and important question that we leave for future work'' and ``leave a more detailed investigation of data dependent weights for future work'' \citep{RinaNonExchange}. So our work is a step towards this direction. 

We further remark on several key differences of \SPCI{} with prior works. Method-wise, our prediction intervals are constructed using conditional quantile regression functions on \textit{non-conformity scores} (e.g., residuals). In contrast, existing quantile-regression-based conformal prediction methods \citep{CPquantile,QOOB} directly fit conditional quantile functions on the \textit{response variables} $Y$, after which the intervals are constructed using \textit{empirical quantiles} of non-conformity scores. Theory-wise, we obtain similar asymptotic conditional coverage for dependent residuals as in \citep{xu2021conformal}. However, different from that work, we do not assume a particular functional form of the conditional distribution of the scalar output given feature variables.

\begin{figure}[!t]
  \centering
  \includegraphics[width=\linewidth]{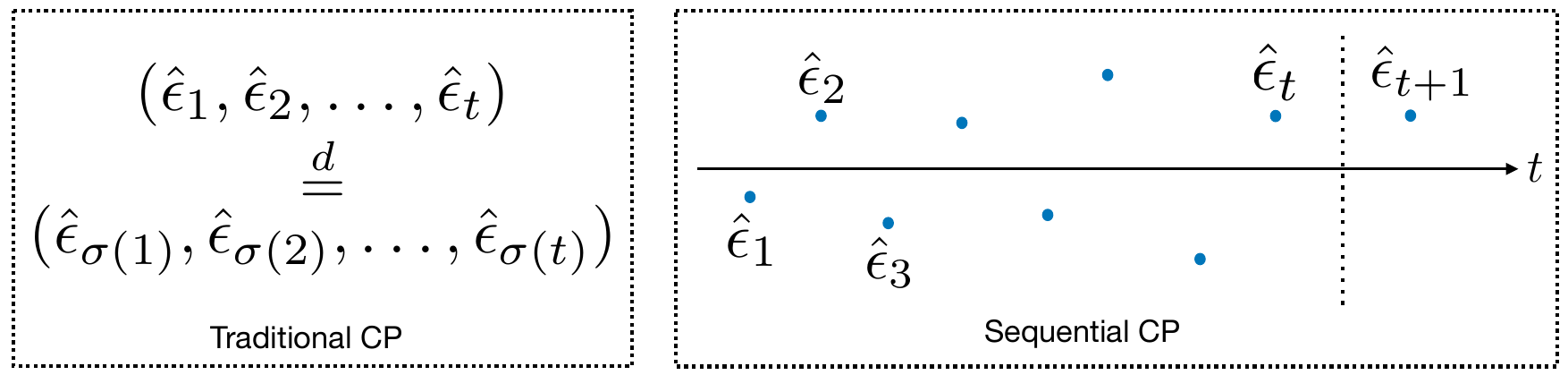}
    \caption{Differences between traditional and sequential Conformal Prediction (CP) methods. 
    In traditional CP, residuals are exchangeable, and the same set of residuals is used throughout the prediction. In contrast, sequential CP assumes an ordering of the potentially non-exchangeable residuals; residuals are available feedback to the prediction algorithms: past residuals are updated to include the new prediction residual $\hat{\epsilon}_{t+1}$ during prediction.}
     \label{compare_resid}
\end{figure}

\vspace{-0.1in}
\section{Problem setup}

Assume a sequence of observations $(X_t, Y_t)$, $t = 1, 2, \ldots$, where $Y_t$ are continuous scalar variables and $X_t \in \R^d$ denote features, which may either be the history of $Y_t$ or contain exogenous variables helpful in predicting the value of $Y_t$.
We can allow observations to be highly correlated under an unknown conditional distribution $Y_t|X_t, \ldots, X_1$, and do not assume a particular functional form of the conditional distribution $Y_t|X_t, \ldots, X_1$. Let the first $T$ samples $\{(X_t,Y_t)\}_{t=1}^T$ be the training data.

Our goal is to construct prediction intervals sequentially starting from time $T+1$ such that the prediction intervals will contain the true outcome with a pre-specified high probability $1-\alpha$ while the prediction interval is as narrow as possible. Here the \textit{significance level} $\alpha$ is user-specified. The prediction intervals $\Ctalpha$, which depend on $\alpha$, are around point predictions $\widehat{Y}_t:=\hat{f}(X_t)$ for a given predictive model $\hat{f}$. A commonly used conformity score is the prediction residual:
\begin{equation*}
\hat{\epsilon}_t=Y_t-\widehat{Y}_t.
\end{equation*}
We emphasize that our algorithm provides prediction intervals for an arbitrary user-chosen predictive algorithm. Here the subscript $_{t-1}$ indicates the interval is constructed using previous up to $t-1$ many observations. 

There are two types of coverage guarantees to be satisfied by $\Ctalpha$. The first is the weaker \textit{marginal} coverage:
\begin{equation}\label{marginal_cov}
    \mathbb P(Y_t \in \Ctalpha) \geq 1-\alpha, \forall t,
\end{equation}
while the second is the stronger \textit{conditional} coverage:
\begin{equation}\label{cond_cov}
    \mathbb P(Y_t \in \Ctalpha|X_t) \geq 1-\alpha, \forall t.
\end{equation}
If $\Ctalpha$ satisfies \eqref{marginal_cov} or \eqref{cond_cov}, it is called marginally or conditionally valid, respectively. In terms of the interval width, to avoid vacuous prediction interval $\Ctalpha$ (in the extreme case, if one chooses the entire real line for all $t$, it will always contain the true outcome $Y_t$ with high probability), we should construct intervals with width $|\widehat{C}_{t-1}(X_t)|$ as narrow as possible. 

A natural approach in developing sequential CP methods is constructing sequential prediction intervals using the most recent feedback in predicting $Y_t$, as shown in Figure \ref{sequential_CP_diagram}. 
However, using the empirical distribution of updated residuals may not fully exploit the temporal dependence across the residuals. 
Indeed, when residuals are temporally correlated, the past residuals contain information about the distribution of future residuals and can be used to perform ``predictive'' conformal inference. More precisely, we should use the past residuals to predict the tail probability of the new residual, as doing so may allow certain adaptivity. The above is the main idea of our proposed \SPCI{} algorithm. 

\begin{figure}[t]
  \centering
  \includegraphics[width=.75\linewidth]{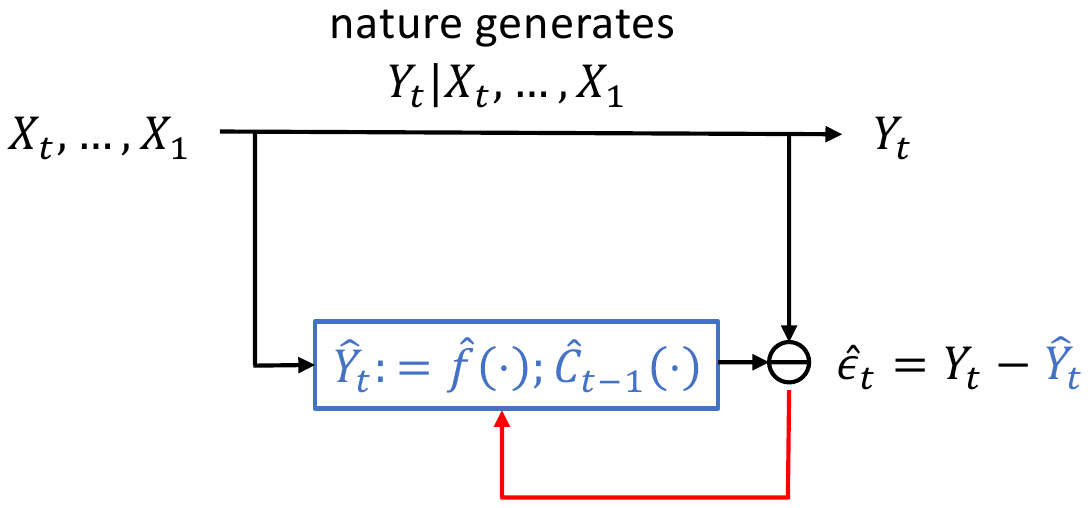}
    \caption{Unlike traditional CP methods, sequential CP methods leverage feedback (in red arrow) during prediction. In this work, we use prediction residual $\hat{\epsilon}_t=Y_t-\widehat{Y}_t$ as an example of the non-conformity score. 
    }
     \label{sequential_CP_diagram}
     \vspace{-0.1in}
\end{figure}

\section{Algorithms}\label{sec_algo}

Below, we first consider a simple split conformal prediction as a vanilla baseline approach based on traditional CP, which constructs prediction intervals without considering feedback during prediction. Then, we present the \EnbPI{} \citep{xu2021conformal} method in sequential CP as a refined approach and illustrate its limitation in using empirical quantile of past residuals. Finally, we introduce the proposed \SPCI{} as an improved algorithm for sequential CP for time series data.

\subsection{Vanilla split conformal}

One of the most commonly used conformal prediction methods is {\it split conformal} \citep{inductCP}, so we describe it as a prototypical example. First, split the indices of training data $[T]:=\{1,\ldots,T\}$ into two halves $\mathcal{I}_1$ and $\mathcal{I}_2$. Second, fit the prediction model $\hat{f}$ on $\{(X_t,Y_t),t\in \mathcal{I}_1\}$ to make point predictions $\widehat{Y}_t=\hat{f}(X_t), t \neq \mathcal{I}_1$. Third, compute \textit{non-conformity score} on $\mathcal{I}_2$, where a typical choice is the residual. Lastly, let $\mathcal{E}[\mathcal{I}_2]=\{\hat{\epsilon}_j\}_{j\in \mathcal{I}_2}$ and define the prediction interval $\Ctalpha$ for $t>T$ as
\begin{equation}\label{traditional_CP}
    [\hat{f}(X_t)+ q_{\alpha/2}(\mathcal{E}[\mathcal{I}_2]),\hat{f}(X_t)+ q_{1-\alpha/2}(\mathcal{E}[\mathcal{I}_2])],
\end{equation}
where $q_{1-\alpha}$ is the $1-\alpha$ quantile function over a set of values. In particular, the set of non-conformity scores $\{\hat{\epsilon}_j\}_{j\in \mathcal{I}_2}$ is fixed during prediction. When $(X_t,Y_t)$ are exchangeable (i.e., we can shuffle the order of these random variables without affecting the joint distribution), split conformal intervals in \eqref{traditional_CP} reaches exact finite-sample marginal coverage defined in \eqref{marg_cov_real}. However, without further distribution assumptions, split conformal intervals cannot reach valid conditional coverage in \eqref{cond_cov} \citep{cond-pi-rina}.

\subsection{EnbPI: Ensemble version using empirical residuals}

Compared to split conformal in the previous section, \EnbPI{} involves no data-splitting, trains ensemble predictors that make more accurate point predictions and utilizes feedback during prediction on test data. Thus, \EnbPI{} is more suitable than split conformal for sequential prediction interval construction. \EnbPI{} has the following three steps. First, it leverages training data as much as possible by fitting ``leave-one-out'' (LOO) ensemble prediction models $\hat{f}_t(X_t):=\phi(\{\hat{f}_b(X_t):t \notin S_b\})$, \revold{where $\phi$ denotes an arbitrary aggregation function (e.g., mean, median, etc.) over a set of scalars,} and $S_b \subset [T]$ is the bootstrap index set used to train the $b$-th bootstrap estimator $\hat{f}_b$. The point predictor on test data is defined as $\hat{f}(X_t):=\phi(\{\hat{f}_b(X_t)\})$, which aggregates all bootstrap predictions. Second, we obtain residuals using the LOO models $\hat{\epsilon}_t:=Y_t-\hat{f}_t(X_t)$. Third, it updates the past residuals during predictions so that the prediction intervals have adaptive width. For a fixed $w\geq 1$, define 
$
    \mathcal{E}_t^w:=\{\hat{\epsilon}_{t-1},\ldots,\hat{\epsilon}_{t-w}\}.
$
Then, \EnbPI{} intervals $\Ctalpha$ have the form:
\begin{equation}\label{EnbPI_interval}
 [\hat{f}(X_t)+ q_{\alpha/2}(\mathcal{E}_t^T),\hat{f}(X_t)+ q_{1-\alpha/2}(\mathcal{E}_t^T)],
\end{equation}
which \revold{utilize the past $w=T$ residuals} and greatly resemble traditional CP intervals in \eqref{traditional_CP} due to the use of empirical quantile function $q_{1-\alpha/2}$ to compute interval width.

However, \EnbPI{} intervals in \eqref{EnbPI_interval} can have limitations under dependent residuals. Note that dependent residuals lead to non-equivalence between conditional and marginal distributions of $\hat{\epsilon}_t$, namely $\hat{\epsilon}_t |\mathcal{E}_t^w \neq \hat{\epsilon}_t$ in distribution. More precisely, let $F(z|\mathcal{E}_t^w):=\bP(\hat \epsilon_t \leq z|\mathcal{E}_t^w)$ be the unknown conditional distribution function of the residual $\hat \epsilon_t$, \revold{ where we implicitly assume the conditional distribution function is invariant over time (i.e., residuals have identical conditional distributions).} Based on \eqref{EnbPI_interval},
\begin{align}
    & \bP(Y_t \in \Ctalpha|X_t)\\
     = & ~\bP(\hat \epsilon_t \in [q_{\alpha/2}(\mathcal{E}_t^T),q_{1-\alpha/2}(\mathcal{E}_t^T)]|X_t)\nonumber \\
     = & ~F(q_{1-\alpha/2}(\mathcal{E}_t^T)|\mathcal{E}_t^w)-F(q_{\alpha/2}(\mathcal{E}_t^T)|\mathcal{E}_t^w). \label{esti_prob}
\end{align}
However, the distribution function $F$ evaluated at the empirical quantiles may not yield the desired coverage. More precisely, define
\begin{equation}\label{true_quantile_Q}
    Q_t(p):=\inf\{e^* \in \R: F(e^*|\mathcal{E}_t^w) \geq p\},
\end{equation}
which is the $p$-th quantile of the residual $\hat \epsilon_t$. By definition,
\begin{equation}\label{true_quantile}
    F(Q_t(1-\alpha/2)|\mathcal{E}_t^w)-F(Q_t(\alpha/2)|\mathcal{E}_t^w)=1-\alpha.
\end{equation}
Thus, in order for \EnbPI{} intervals in \eqref{EnbPI_interval} to have the desired $1-\alpha$ coverage asymptotically, the empirical quantile must uniformly converge to the actual quantile value, namely:
\begin{equation}\label{quantile_closeness}
\sup_{p \in [0,1]}|q_{p}(\mathcal{E}_t^T)-Q_t(p)|\rightarrow 0 \text{ as } T\rightarrow \infty.
\end{equation}
However, the condition \eqref{quantile_closeness} requires strong assumptions: \citep{xu2021conformal} assumes a particular linear functional form of $Y_t|X_t$ (i.e., $Y_t=f(X_t)+\epsilon_t)$, which further needs to be consistently estimated as sample size approaches infinity. Such assumptions can impose limitations in practice.


\subsection{Proposed \SPCI\ algorithm}

Due to the limitations above by split conformal and \EnbPI, we propose \SPCI{} in Algorithm \ref{algo_SPCI} as a more general framework than both approaches. In particular, \SPCI{} directly leverages the dependency of $\hat{\epsilon}_t$ on the past residuals when constructing the prediction intervals. Based on the equivalence in \eqref{esti_prob} and the coverage property in \eqref{true_quantile}, \SPCI{} replaces the empirical quantile with an estimate by a conditional quantile estimator. Specifically, let $\hatQtz{p}$ be an estimator of the true quantile $Q_t(p)$ in \eqref{true_quantile_Q} and let $\hat{f}$ be a pre-trained point predictor, \SPCI{} intervals $\Ctalpha$ are defined as
\begin{equation}\label{SPCI_interval}
    [\hat{f}(X_t)+\hatQtz{\betaReal},\hat{f}(X_t)+ \hatQtz{1-\alpha+\betaReal}],
\end{equation}
where $\hat{\beta}$ minimizes interval width:
\begin{equation}\label{beta_hat}
    \betaReal = {\arg\min}_{\beta\in[0,\alpha]} (\hatQtz{1-\alpha+\beta}-\hatQtz{\beta}).
\end{equation}
In particular, \SPCI{} is more general than both \EnbPI{} and split conformal. If we train LOO point predictors, choose the quantile estimator $\hatQtz{\cdot}$ as the empirical quantile, and use $\betaReal=\alpha/2$, \SPCI{} in \eqref{SPCI_interval} reduces to \EnbPI{} in \eqref{EnbPI_interval}. If we follow split conformal prediction to train the point predictor $\hat{f}$, train quantile predictor $\hatQt$ on residuals from calibration set, and do no update residuals during prediction, \SPCI{} intervals reduce to the split conformal intervals in \eqref{traditional_CP}.


We particularly comment on the computational aspect of fitting conditional quantile estimators $\hatQt$, the essential step of \SPCI. To train $\hatQt$, one minimizes the pinball loss
\begin{equation}\label{quantile_loss}
    \mathcal{L}(x,\alpha)=
    \begin{cases}
    \alpha x & \text{if } x\geq 0,\\
    (\alpha-1) x & \text{if } x< 0,
    \end{cases}
\end{equation}
which depends on the significance level $\alpha$. Because \SPCI{} aims to produce intervals as narrow as possible and refits the quantile regression models at each $t$, it is important to choose quantile regression algorithms that are efficient enough in this sequential setting. In this work, we will use quantile random forest (QRF) \citep{Meinshausen2006QuantileRF} to train $\hatQt$ and establish coverage guarantees.


\begin{algorithm}[t]
\cprotect\caption{Sequential Predictive Conformal Inference (\SPCI)}
\label{algo_SPCI}
\begin{algorithmic}[1]
\REQUIRE{Training data $\{(X_t, Y_t)\}_{t=1}^T$, 
prediction algorithm $\mathcal{A}$,
significance level $\alpha$, quantile regression algorithm $\mathcal{Q}$ .
}
\ENSURE{Prediction intervals $\widehat{C}_{t-1}(X_t), t>T$}
\STATE Obtain $\hat{f}$ and  \textit{prediction} residuals $\widehat{\boldsymbol \epsilon}$ with $\mathcal{A}$ and $\{(X_t, Y_t)\}_{t=1}^T$
\FOR {$t>T$}
\STATE Use quantile regression to obtain $\hatQt\leftarrow\mathcal{Q}(\widehat{\boldsymbol \epsilon})$
\STATE Obtain prediction interval $\Ctalpha$ as in \eqref{SPCI_interval}
\STATE Obtain new residual $\hat \epsilon_t$
\STATE Update residuals $\widehat{\boldsymbol \epsilon}$ by sliding one index forward (i.e., add $\hat \epsilon_t$ and remove the oldest one)
\ENDFOR
\end{algorithmic}
\end{algorithm}

We train QRF \textit{auto-regressively} in \SPCI{} to leverage the dependency in residuals. In short, we use the past $w\geq 1$ residuals to predict the conditional quantile of the future (unobserved) residual. More precisely, suppose we have $T$ past residuals $\mathcal{E}_t^T$ available at prediction index $t$. Let $\tilde{T}:=T-w$. For $t'=1,\ldots, \tilde{T}$, define
\begin{equation}\label{QRF_data}
    \tilde{X}_{t'} :=[\hat{\epsilon}_{t'+w-1},\ldots,\hat{\epsilon}_{t'}],
    \tilde Y_{t'} := \hat{\epsilon}_{t'+w}. 
\end{equation} 
Thus, feature $\tilde X_{t'}$ contains $w$ residuals useful for predicting the conditional quantile of $\tilde Y_{t'}$, which is the residual at index ${t'}+w$. We use the feature $\tilde X_{\tilde T+1}$ to predict the conditional quantile of $\tilde Y_{\tilde T+1}$. As a result, the QRF is trained using $\Tilde{T}$ training data $(\tilde X_{t'}, \tilde Y_{t'}), t'=1,\ldots, \tilde T$. When re-fitting the QRF at each prediction index, we re-design these $\tilde T$ training data using a sliding window of most recent $T$ residuals. In our experiments, we use the Python implementation of QRF by \citep{sklearn_quantile}.

\rev{
\begin{remark}[\SPCI{} vs. Quantile regression]
    We further highlight the essential difference of \SPCI{} against quantile regression approaches.
    In general, quantile regression algorithms rely on minimizing the pinball loss for specific regression algorithms. Doing so can often lead to inaccurate results and require special hyper-parameter tuning. In general, these algorithms can also be computationally expensive to train for multiple significance levels, as the pinball loss depends on $\alpha$.
    In contrast, \SPCI{} is compatible with any user-specified point prediction model, remains distribution-free, and provides coverage guarantees (see Theorem \ref{thm_non_exchange}). Hence, \SPCI{} inherits the main benefits of CP methods. In addition, \SPCI{} leverages the dependency of non-conformity scores by fitting a QRF model of the quantiles (one can use general quantile regression models if desired). Computationally, \SPCI{} is also efficient in test time as the fitting of the QRF model does not rely on the significance level alpha.
    In practice, we find such a hybrid approach to outperform quantile regression models using deep neural networks (see Table \ref{marg_cov_real}). 
\end{remark}
}

\section{Theory}
We first show that when data are exchangeable, one can reach exact marginal coverage when using the empirical quantile function as the quantile regression predictor. We then establish asymptotic coverage upon considering the dependency of estimated residuals. \revold{For dependent residuals, we adapt the proof in \citep{Meinshausen2006QuantileRF} for independent observations, where we replace the independence assumption with stationary and decaying dependence assumptions.} Most proofs and additional theoretical details appear in Appendix \ref{proof}. 

\subsection{Under exchangeability}

\rev{We show below that \SPCI{} maintains marginal coverage when data are exchangeable. The proof is standard based on showing the marginal coverage of split conformal prediction}

\begin{proposition}[Finite-sample marginal coverage under exchangeability {\citep{inductCP}}]\label{dist_free}
Suppose the data $(X_t,Y_t), t\geq 1$ are exchangeable (e.g., independent and identically distributed). Prediction intervals obtained via Algorithm \ref{algo_SPCI_exchangeable} (i.e., a special version of \SPCI{}) satisfy
\[
\mathbb P(Y_t \in \widehat{C}_{t-1}(X_t)) \geq 1-\alpha.
\]
\end{proposition}

\subsection{Beyond exchangeability}\label{theory_spci_non_exch}

The primary theoretical contribution of our work is to show the asymptotic conditional validity of \SPCI{} intervals when the quantile random forest \citep{Meinshausen2006QuantileRF} is used as the conditional quantile estimator. Specifically, we have
\[
\bP(Y_t \in \Ctalpha|X_t) \rightarrow 1-\alpha \text{ as } T\rightarrow \infty, 
\]
which by \eqref{esti_prob} and \eqref{true_quantile_Q}, is equivalent to proving
\begin{equation}\label{convergence}
    \sup_{p\in [0,1]}|\hatQtz{p}-Q_t(p)|\rightarrow 0 \text{ as } T\rightarrow \infty,
\end{equation}
where $\hatQtz{p}$ is the QRF estimator. More precisely, we want to estimate the conditional quantile values of $\tilde{Y}_{\Tilde{T}+1}$ given $\tilde{X}_{\Tilde{T}+1}$, both of which are defined in \eqref{QRF_data}. Note that \eqref{convergence} for i.i.d. observations has been proven in \citep[Theorem 1]{Meinshausen2006QuantileRF}, so that our analysis also extends the original statement therein to observations with dependency.

We follow the notation in \citep{Meinshausen2006QuantileRF} to introduce QRF. For the feature $\tilde{X}_t, t\geq 1$, assume its support $\text{Supp}(\tilde{X}_t)\subset \mathbb B \subset \R^p$. We grow the tree $T(\theta)$ with parameter $\theta$ as follows: every leaf $l=1,\ldots,L$ of a tree $T(\theta)$ is associated with a rectangular subspace $R_l \subset \mathbb B$. In particular, they are disjoint and cover the entire space $\mathbb B$: for every $x\in \mathbb B$, there is \textit{one and only one} leaf $l$, thus denoted as $l(x,\theta)$, such that $x\in R_{l(x,\theta)}$. If we grow $K$ trees, let each of them have separate parameter $\theta_k$. Now, for a given $x\in \mathbb B$ and $\tilde{T}$ observed features $\tilde X_1,\ldots,\tilde X_{\tilde T}$, we define the following weights:
\begin{align}
    k_{\theta}(l) &:= \# \{j \in \{1,\ldots, \tilde T\}: \tilde X_j \in R_{l(x,\theta)}\} \label{k_theta_l} \\
     w_t(x,\theta) &:=\frac{\boldone(\tilde X_t \in R_{l(x,\theta)})}{k_{\theta}(l)} \label{w_t_tree} \\
    w_t(x) &:=K^{-1} \sum_{k=1}^K w_t(x,\theta_k) \label{w_t_forest}
\end{align}
For interpretation, \eqref{k_theta_l} counts the ``node size'' of the leaf $l(x,\theta)$, \eqref{w_t_tree} weighs the $i$-th observation using whether $\tilde X_t$ belongs to this leaf and its node size, and \eqref{w_t_forest} weighs such weights from $K$ trees. Based on weights in \eqref{w_t_forest}, the estimated conditional distribution function $\hat F(z|x)\revold{=\hat F(z|\tilde{X}_{\tilde{T}+1}=x)}$ is defined as 
\begin{equation}\label{QRF}
   \hat F(z|x):= \sum_{t=1}^{\tilde{T}} w_t(x) \boldone(\tilde Y_t\leq z).
\end{equation}
In retrospect, the estimation in \eqref{QRF} is similar to that under fixed weights by \citep{RinaNonExchange}. The key difference is that \eqref{QRF} uses data-adaptive weights \revold{as it exploits the temporal autocorrelation of residuals}. In contrast, \citep{RinaNonExchange} uses fixed and non-adaptive weights.

To show the convergence of the estimated QRF quantile to the true value, we first have the following lemma relating the convergence of quantile estimates to the convergence of corresponding distribution functions.

\begin{lemma}\label{lem1}
    For random variable $\hat \epsilon_t$ (i.e., residual in our setup), let $F(z|x)$ be its conditional distribution function and $Q(p):=\inf \{z\in \R: F(z|x)\geq p\}$ be the $p$-th quantile, which is assumed to be unique. Let $\hat F(z|x)$ be an estimator trained on $\tilde T$ samples $\{(\tilde X_t,\tilde Y_t)\}_{t=1}^{\tilde T}$. If for all $z$ and $x$ it holds that 
    \begin{equation}\label{ptwise_converge}
        \hat F(z|x) \rightarrow F(z|x) \text{ in probability as } \tilde T \rightarrow \infty,
    \end{equation}
    then $\widehat Q(p):=\inf \{z\in \R: \hat F(z|x)\geq p\}$ satisfies 
    $\widehat Q(p) \rightarrow  Q(p)$ in probability for every $p\in (0,1)$ and $x$.
\end{lemma}
Thus, the crux of the remaining analyses relies on showing the point-wise convergence in \eqref{ptwise_converge} for the QRF in \eqref{QRF}. The case where all data are independent and identically distributed has been addressed in \citep[Theorem 1]{Meinshausen2006QuantileRF}. We address the more general case for dependent observations in Proposition \ref{ptwise_converge_QRF}.

\begin{proposition}\label{ptwise_converge_QRF}
    If Assumptions \ref{Assume0}---\ref{Assume3} defined in Appendix \ref{proof} hold, we obtain the point-wise convergence in \eqref{ptwise_converge} for QRF.
\end{proposition}

We briefly explain and discuss the necessary theoretical assumptions \ref{Assume0}---\ref{Assume3} used in proving Proposition \ref{ptwise_converge_QRF}:
\begin{itemize}
    \item Assumption \ref{Assume0}: This assumption states two things. First, the dependency of the covariance of the indicator random variables (defined over the residual quantiles) only depends on the difference in index (see Eq. \eqref{cov}). Such assumption on residual dependency resembles the weak or wide-sense stationarity assumption. Second, the value of covariances can be uniformly bounded over the conditioning variables by a function $\tilde{g}$, and there is a growth order constraint on $\tilde{g}$ (see Eq. \eqref{g_growth}). This condition is imposed to avoid strong dependency among the residuals, which prevents asymptotic consistency of the QRF estimator.
    
    \item Assumption \ref{Assume1}: This assumption requires that the weights $w_t(x)$ in QRF decay linearly with respect to the number of training samples for QRF. In practice, we often found that the weights decay at such an order.
    
    \item Assumption \ref{Assume2} and \ref{Assume3}: These distributional assumptions on the conditional quantile function follow those in QRF \citep{Meinshausen2006QuantileRF}, and they are reasonably mild. In particular, we are not assuming a particular parametric of the conditional quantile function so the results are distribution-free.
\end{itemize}

We finally obtain the asymptotic guarantee on interval coverage.

\begin{theorem}[Asymptotic conditional coverage beyond exchangeability]\label{thm_non_exchange}
Under the same assumptions as Lemma \ref{lem1} and Proposition \ref{ptwise_converge_QRF}, as the sample size $T\rightarrow \infty$, we have for any $\alpha \in (0,1)$
\begin{equation}\label{coverage}
    |\bP(Y_t \in \widehat{C}_{t-1}(X_t)|X_t)-(1-\alpha)| \overset{p}{\rightarrow} 0
\end{equation}
\end{theorem}

\subsection{Implications of results}
We discuss several implications of the results: (1) how the results are distribution-free and model-free; (2) challenges in obtaining interval convergence; (3) the generality of proving guarantees for QRF; (4) convergence analyses beyond using QRF. 

\vspace{0.1in}
\noindent \textit{Distribution-free \& model-free guarantees.}
Note that our coverage guarantee makes no explicit distributional assumptions of the residuals (e.g., density function has certain parametric form). Instead, our assumptions are on the dependency among the residuals and the regularity of the density functions of the residuals. 
On the other hand, our results also make no assumptions on the underlying data generation process of $Y_t$ given $X_t$, in contrast to the linear assumption $Y_t=f(X_t)+\epsilon_t$ in \EnbPI{} \citep{xu2021conformal}. One can also use arbitrary predictive model to obtain the residuals, rather than relying on special deep neural network architectures \citep{salinas2020deepar,lim2021temporal}.

\vspace{0.1in}
\noindent \textit{Interval convergence.} Ideally, we wish \SPCI{} intervals in \eqref{SPCI_interval} to converge in width to the oracle interval defined by $Y_t|X_t$. However, doing so requires assumptions on the inverse CDF of $Y_t|X_t$, which deviate from our focus on model-free interval construction. Even though such theoretical analyses are lacking, experiments in Section \ref{sec_exper} demonstrate that \SPCI{} improves over recent sequential conformal prediction models in many cases.

\vspace{0.1in}
\noindent \textit{Generality of QRF.}
Note that decision trees are simple functions, thus satisfying the assumptions of the \textit{Simple Function Approximation Theorem} \citep{royden1988real}. In other words, the QRF estimates can theoretically approximate those of any other quantile estimates. As a result, this can be useful if one analyzes the convergence of QRF quantile estimates for residuals with a more general dependency.  

\vspace{0.1in}
\noindent \textit{Convergence beyond using QRF.}
The convergence of quantile estimates has been a long-standing question in statistics. In our case, we are particularly interested in the quantile estimates under time-series data. In the past, several lines of work have established such results for different estimators under various assumptions on dependency. \citep{cai2002regression} studied weighted Nadaraya-Watson quantile estimates for $\alpha$-mixing sequences. \citep{biau2011sequential} proposes a nearest-neighbor strategy for stationary and ergodic data. \citep{zhou2009local} analyzed local linear quantile estimators for locally stationary time series. More analyses appear in the survey \citep{xiao2012time}.

\section{Experiments}\label{sec_exper}

We empirically demonstrate the improved performance of \SPCI{} over competing sequential CP methods \rev{and probabilistic forecasting methods} in terms of interval coverage and width. The CP baselines are \EnbPI{} \citep{xu2021conformal}, AdaptiveCI \citep{Gibbs2021AdaptiveCI}, and NEX-CP \citep{RinaNonExchange}, whose details are in Appendix \ref{append_exp}. 
The two probabilistic forecasting methods are DeepAR \citep{salinas2020deepar} and TFT \citep{lim2021temporal}. 
In all experiments, we obtain LOO point predictors $\hat{f}$ and prediction residuals $\widehat{\boldsymbol \epsilon}$ as in \EnbPI{}. Official implementation can be found at \url{https://github.com/hamrel-cxu/SPCI-code}.

\subsection{Simulation}\label{sec:simulation}

We \revold{first} compare \SPCI{} with \EnbPI{} on non-stationary and/or heteroskedastic time-series. \revold{We then compare \SPCI{} with NEX-CP on data with distribution drifts and change-points under the setting described in \citep{RinaNonExchange}.} Details on data simulation are in Appendix \ref{append_simul}.

\vspace{0.1in}
\noindent \revold{\textit{(1) Comparison with \EnbPI{}}.} Given a feature $X_t$, we specify the true data-generating process as 
$Y_t=f(X_t)+\epsilon_t.$
We simulate two types of time-series data. The first considers non-stationary (Nstat) time-series. The second considers heteroskedastic (Hetero) time-series in which the variance of $\epsilon_t$ depends on $X_t$.

Table \ref{simul_results} compares \EnbPI{} with \SPCI{}, where both use the random forest regression model to fit the point estimator $\hat{f}$. We see clear improvement of \SPCI{}. We suspect the improvement lies in the more adaptive and accurate calibration of quantile values of residual distributions in prediction.

\begin{table}[!t]
    \caption{Simulation: \texttt{EnbPI} vs. \texttt{SPCI} on simulated time-series with $\alpha=0.1$. \texttt{SPCI} outperforms \texttt{EnbPI} in terms of interval width without sacrificing valid coverage.}
    \label{simul_results}
    \centering
    \resizebox{\linewidth}{!}{
    \begin{tabular}{lllll}
    \toprule
    {} &    Nstat coverage &       Nstat width & Hetero coverage &   Hetero width\\
    \midrule
    SPCI  &                 0.94 (2.04e-3) &             11.23 (3.37e-2) &                   0.89 (9.43e-3) &               24.09 (8.27e-1) \\
    EnbPI &                 0.91 (1.11e-3) &             25.22 (2.84e-2) &                   0.92 (1.18e-2) &               25.84 (3.47e-1) \\
    \bottomrule
    \end{tabular}
    }
    \vspace{-0.05in}
\end{table}

\begin{table}[!b]
    \caption{\revold{Simulation: NEX-CP vs. \texttt{SPCI} on simulated time-series with 90\% target coverage. Entries in the bracket indicate standard deviation over ten trials where data are re-generated. The symbol * denotes results from \citep[Table 1]{RinaNonExchange}. Results from the second row are based on $\alpha=0.09$ (dist. shift) and $\alpha=0.075$ (change-point).
    }}
    \label{simul_results_with_NEXCP}
    \centering
    \resizebox{\linewidth}{!}{
    \begin{tabular}{lllll}
    \toprule
    {} &    Drift coverage &       Drift width & Change coverage &   Change width\\
    \midrule
    SPCI  &                 0.89 (5.04e-3) &             3.33 (4.17e-2) &                   0.87 (2.75e-3) &               3.85 (4.12e-2) \\
    SPCI, adjusted $\alpha$ &                 0.90 (4.63e-3) &             3.43 (4.43e-2) &                   0.90 (3.71e-3) &               4.18 (4.89e-2) \\
    NEX-CP* &                 0.91  &             3.45 &                   0.91 &               4.13 \\
    \bottomrule
    \end{tabular}
    }
\end{table}

\vspace{0.1in}
\revold{\noindent \textit{(2) Comparison with NEX-CP}. We consider data with distribution drift and changepoints, where data are simulated according to examples in \citep{RinaNonExchange}. 

Table \ref{simul_results_with_NEXCP} shows competitive results of both methods. We notice slight under-coverage by \SPCI{} under both settings, despite the much narrower intervals by \SPCI{}. When we slightly lower the significance level $\alpha$, \rev{which is held constant when constructing all intervals,} \SPCI{} maintains valid coverage with comparable interval widths as NEX-CP.
Figure \ref{rolling_cov_simul} visualizes rolling coverage and width after a burn-in period, with a rolling window of 50 samples.
The results are similar to the best model in \citep[Figure 2]{RinaNonExchange}. In Appendix \ref{append_simul}, we further explain why \SPCI{} tends to under-cover in these settings before $\alpha$ adjustment.
}

\begin{figure}[!b]
    \vspace{-0.05in}
    \begin{minipage}{\linewidth}
        \centering 
        \includegraphics[width=\textwidth]{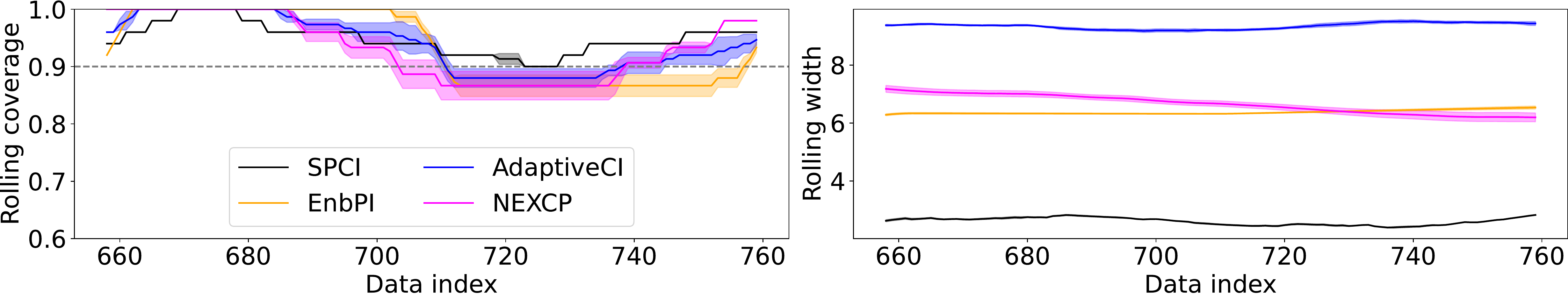}
    \subcaption{Wind}
    \end{minipage}
    \begin{minipage}{\linewidth}
        \centering 
        \includegraphics[width=\textwidth]{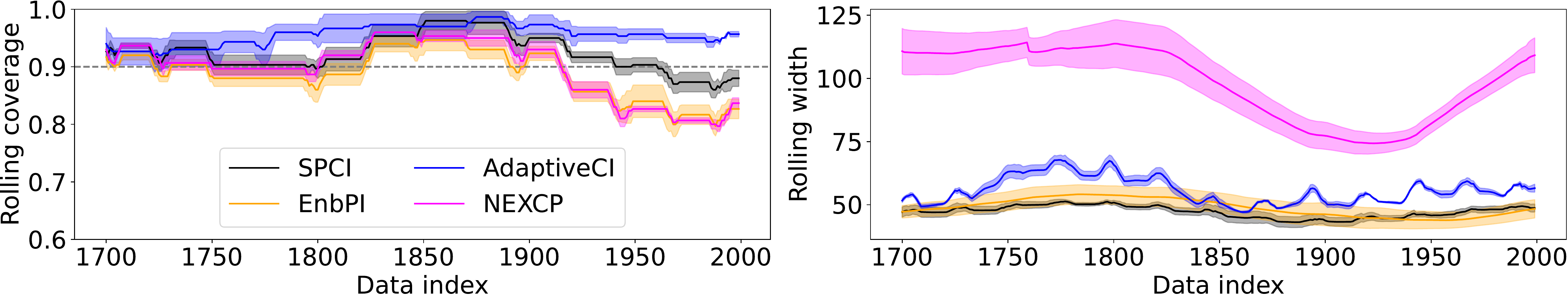}
        \subcaption{Solar}
    \end{minipage}
    \begin{minipage}{\linewidth}
        \centering 
        \includegraphics[width=\textwidth]{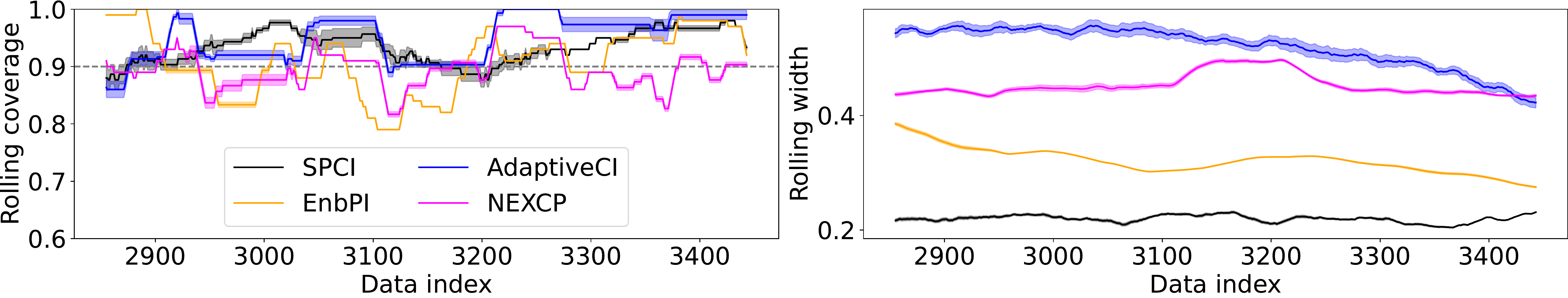}
        \subcaption{Electric}
    \end{minipage}
    \vspace{-0.1in}
    \caption{Rolling coverage and interval width over three real time series by different methods. \SPCI{} in black not only yields valid rolling coverage but also consistently yields the narrowest prediction intervals. Furthermore, the variance of \SPCI{} results over trials is also small, as shown by the shaded regions over coverage and width results.}
    \label{rolling_cov_real}
\end{figure}

\begin{table*}[!t]
    \caption{Marginal coverage and width by all methods on three real time series. The target coverage is 0.9, and entries in the bracket indicate standard deviation over three independent trials. \SPCI{} outperforms competitors with a much narrower interval width and does not lose coverage.}
    \label{marg_cov_real}
    \centering
    \resizebox{\textwidth}{!}{
    \begin{tabular}{lllllll}
    \toprule
    {} &    Wind coverage &       Wind width & Electric coverage &   Electric width &   Solar coverage &        Solar width \\
    \midrule
    SPCI       &  0.95 (1.50e-2) &  2.65 (1.60e-2) &   0.93 (4.79e-3) &  0.22 (1.68e-3) &  0.91 (1.12e-2) &   47.61 (1.33e+0) \\
    EnbPI      &  0.93 (6.20e-3) &  6.38 (3.01e-2) &   0.91 (6.84e-4) &  0.32 (9.11e-4) &  0.88 (4.25e-3) &   48.95 (3.38e+0) \\
    AdaptiveCI &  0.95 (5.37e-3) &  9.34 (3.56e-2) &   0.95 (1.81e-3) &  0.51 (7.25e-3) &  0.96 (1.39e-2) &   56.34 (1.15e+0) \\
    NEX-CP      &  0.96 (8.21e-3) &  6.68 (7.73e-2) &   0.90 (2.05e-3) &  0.45 (2.16e-3) &  0.90 (7.73e-3) &  102.80 (5.25e+0) \\
    \hline
    DeepAR &  0.95 (5.32e-3) & 6.86 (7.86e-3) & 0.91 (3.45e-3) & 0.62 (2.56e0-3) & 0.92 (5.35e-3) & 80.23 (4.94e+0) \\
    TFT      &  0.92 (6.34e-2) & 7.56 (5.34e-3)
 & 0.95 (2.34e-2)& 0.66 (2.34e-3)	 & 0.93 (2.84e-3) & 74.82 (4.23e+0) \\
    \bottomrule
    \end{tabular}
    }
    \vspace{-0.05in}
\end{table*}

\subsection{Real-data examples}\label{sec_real_data}

We primarily consider three real time-series in this section, whose details are in Appendix \ref{append_exp}. We first compare the marginal coverage and width of \SPCI{} against baseline methods. We then examine the rolling coverage and width of each method to assess their stability during prediction. We lastly apply \SPCI{} on a more challenging multi-step ahead inference case to illustrates its usefulness. We fix $\alpha=0.1$ and use the first 80\% (resp. rest 20\%) data for training (resp. testing). For \SPCI{} and \EnbPI, we use the random forest regression model with 25 bootstrap models.

\vspace{0.1in}
\noindent \textit{(1) Marginal coverage and width.} Table \ref{marg_cov_real} shows the marginal coverage and width of all methods on the three time series. While all methods nearly maintain validity at $\alpha=0.1$, \SPCI{} yields significantly narrower intervals, especially on the wind speed prediction data. Such results illustrate the advantages of fitting conditional quantile regression on residuals for width calibration and training LOO regression predictors for point prediction. 

\rev{In the appendix, Table \ref{stock_results} further compares \SPCI{} against AdaptiveCI on stock market return data, which are similar to ones used in \citep{Gibbs2021AdaptiveCI}. We show that \SPCI{} always maintains valid $1-\alpha$ coverage and yields narrower intervals than Adaptive CI.}

\vspace{0.1in}
\noindent \textit{(2) Rolling coverage and width.} Besides the marginal metric, we provide further insights into the dynamics of prediction intervals. Figure \ref{rolling_cov_real} visualizes the rolling coverage and width of each method, where the metric is computed over a rolling window of size 100 (resp. 50) for the solar and electricity (resp. wind) datasets. The results first show that \SPCI{} barely loses rolling coverage when competing methods (e.g., \EnbPI) can fail to do so. Secondly, \SPCI{} intervals are adaptive: they are wider or narrower depending on the data index, which likely reflects higher or less uncertainty in test data. Thirdly, \SPCI{} intervals are evidently narrower than those by competing methods. Lastly, \SPCI{} rolling results have less variance than others such as NEX-CP.

\begin{figure}[!t]
    \centering
    \begin{minipage}{0.48\linewidth}
        \includegraphics[width=\textwidth]{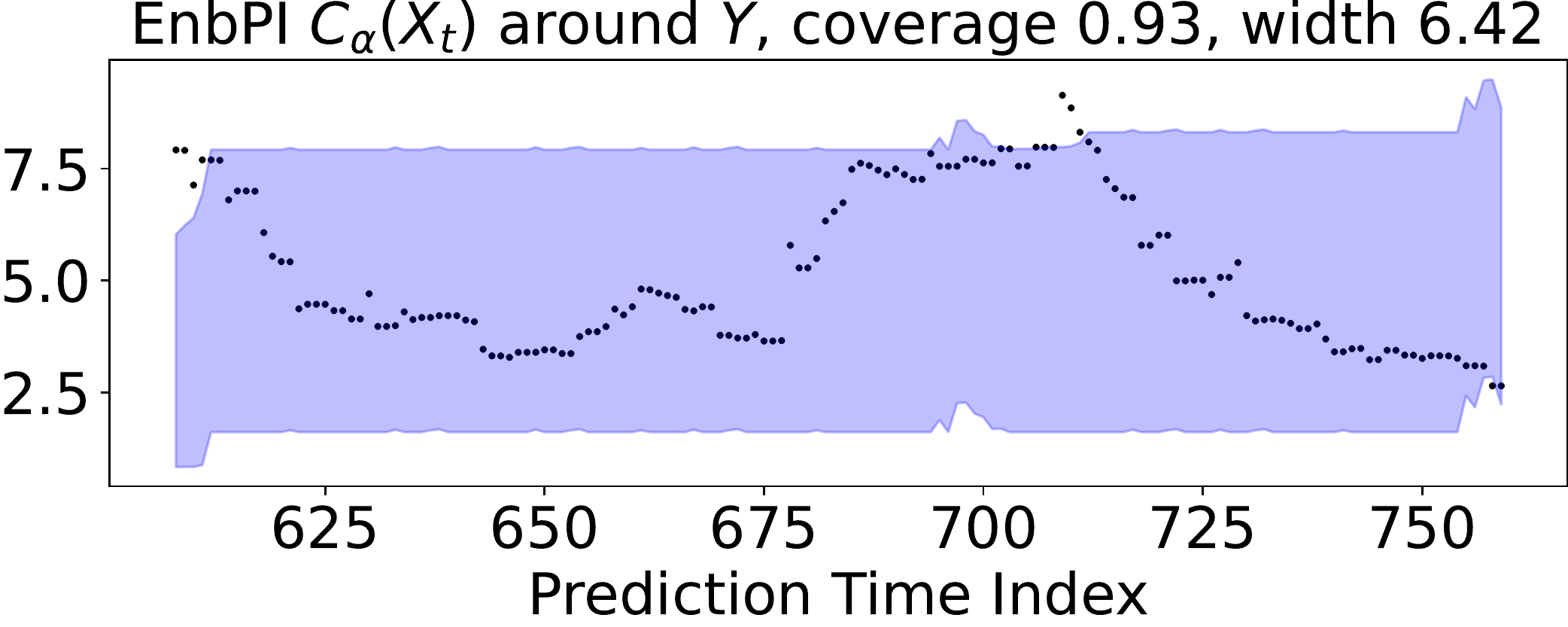}
        \subcaption{EnbPI, 1 step ahead}
    \end{minipage}
    \begin{minipage}{0.48\linewidth}
        \centering \includegraphics[width=\textwidth]{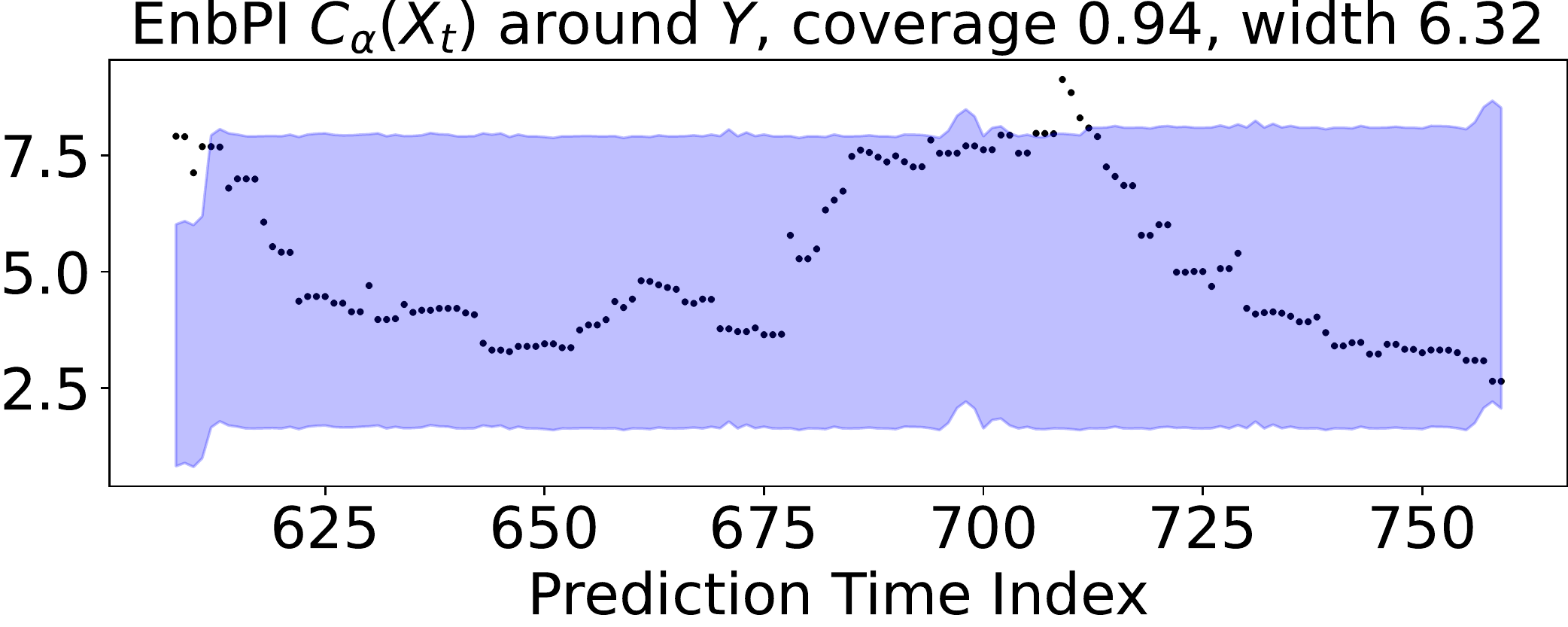}
        \subcaption{EnbPI, 4 steps ahead}
    \end{minipage}
    
    \begin{minipage}{0.48\linewidth}
       \centering \includegraphics[width=\linewidth]{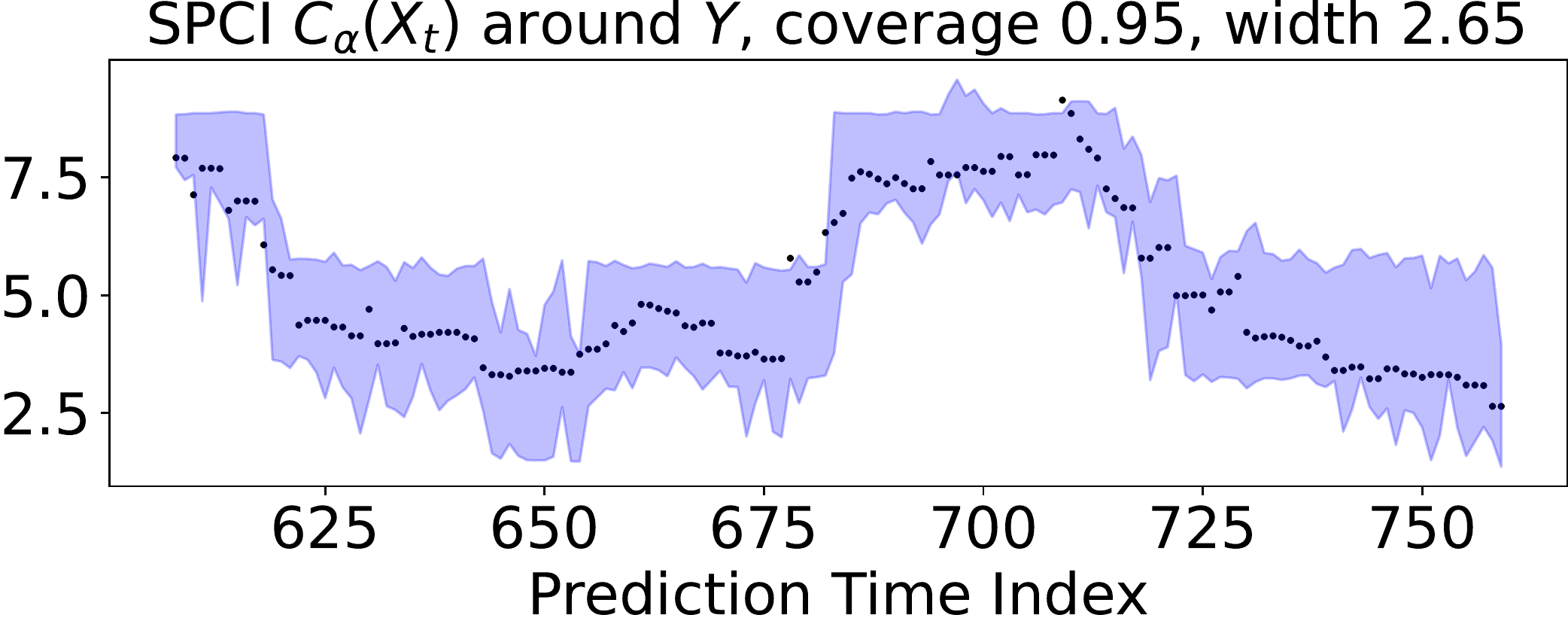}
        \subcaption{SPCI, 1 step ahead}
    \end{minipage}
    \begin{minipage}{0.48\linewidth}
       \centering \includegraphics[width=\linewidth]{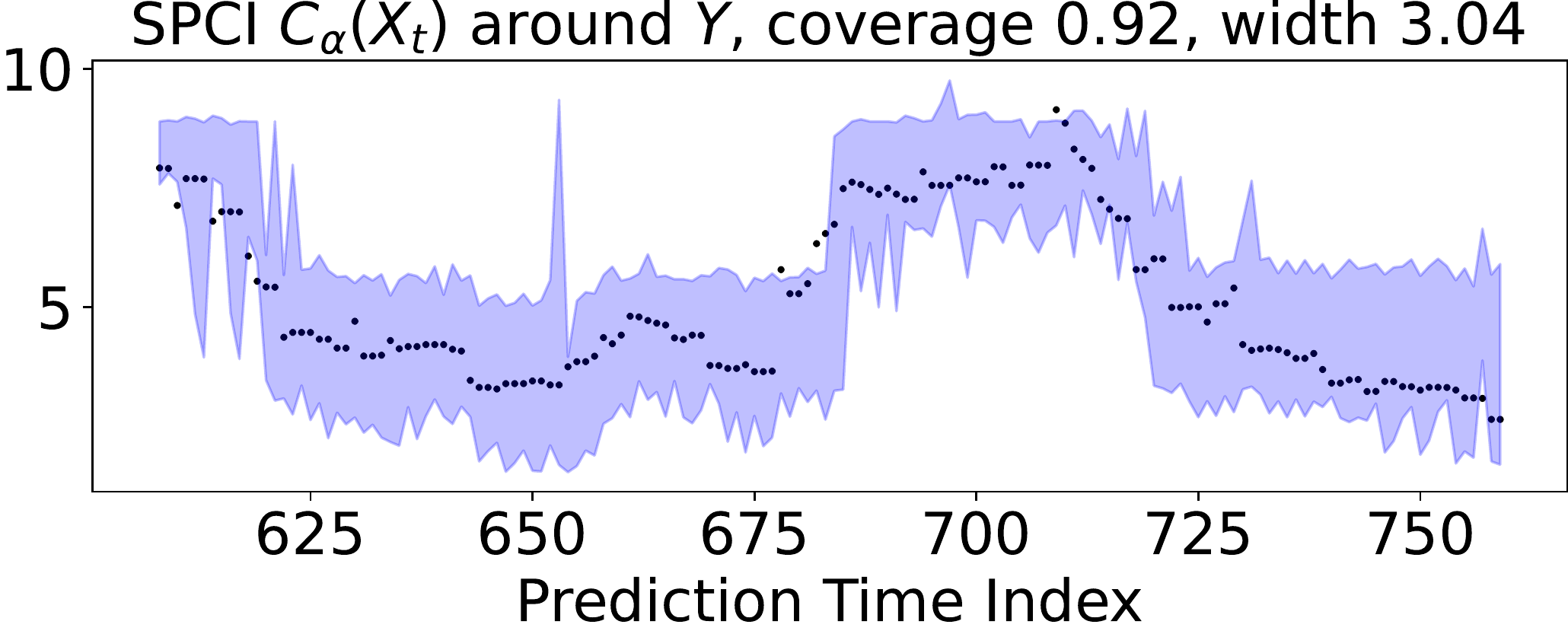}
        \subcaption{SPCI, 2 step ahead}
    \end{minipage}
    
    \begin{minipage}{0.48\linewidth}
       \centering \includegraphics[width=\linewidth]{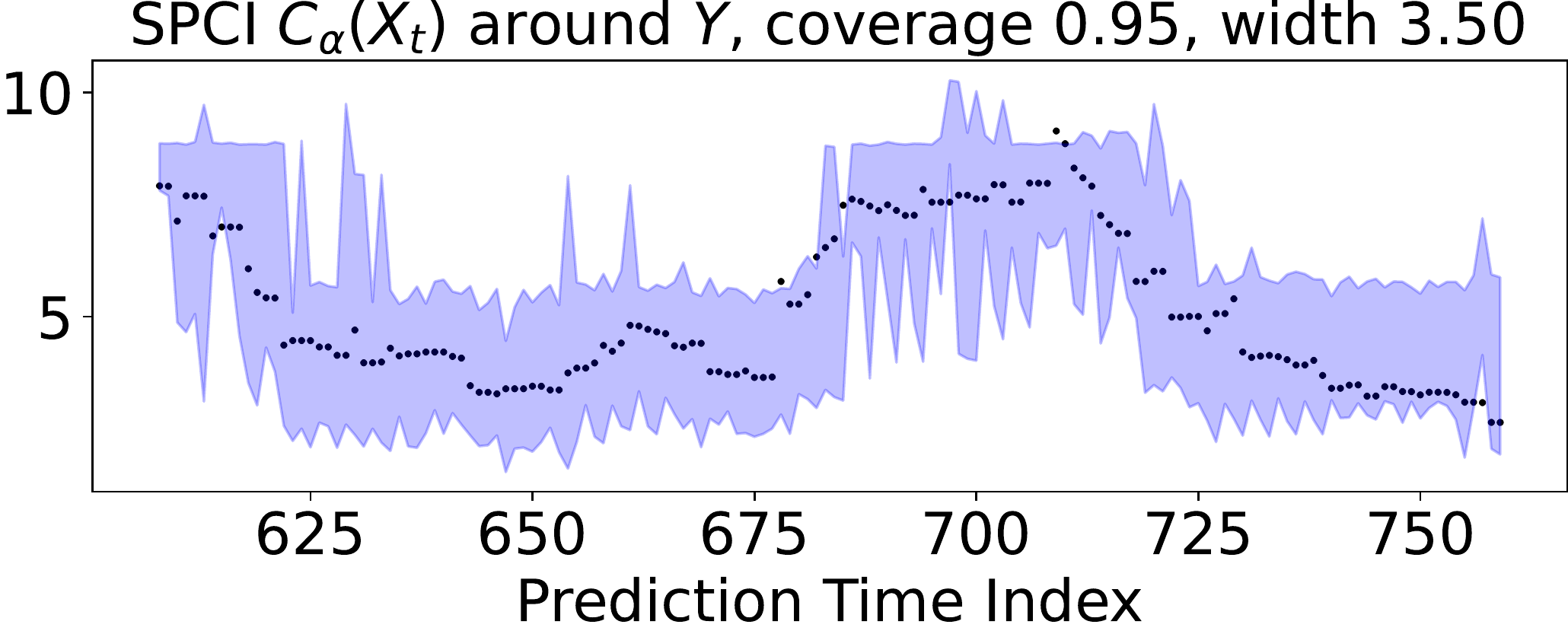}
        \subcaption{SPCI, 3 step ahead}
    \end{minipage}
    \begin{minipage}{0.48\linewidth}
       \centering \includegraphics[width=\linewidth]{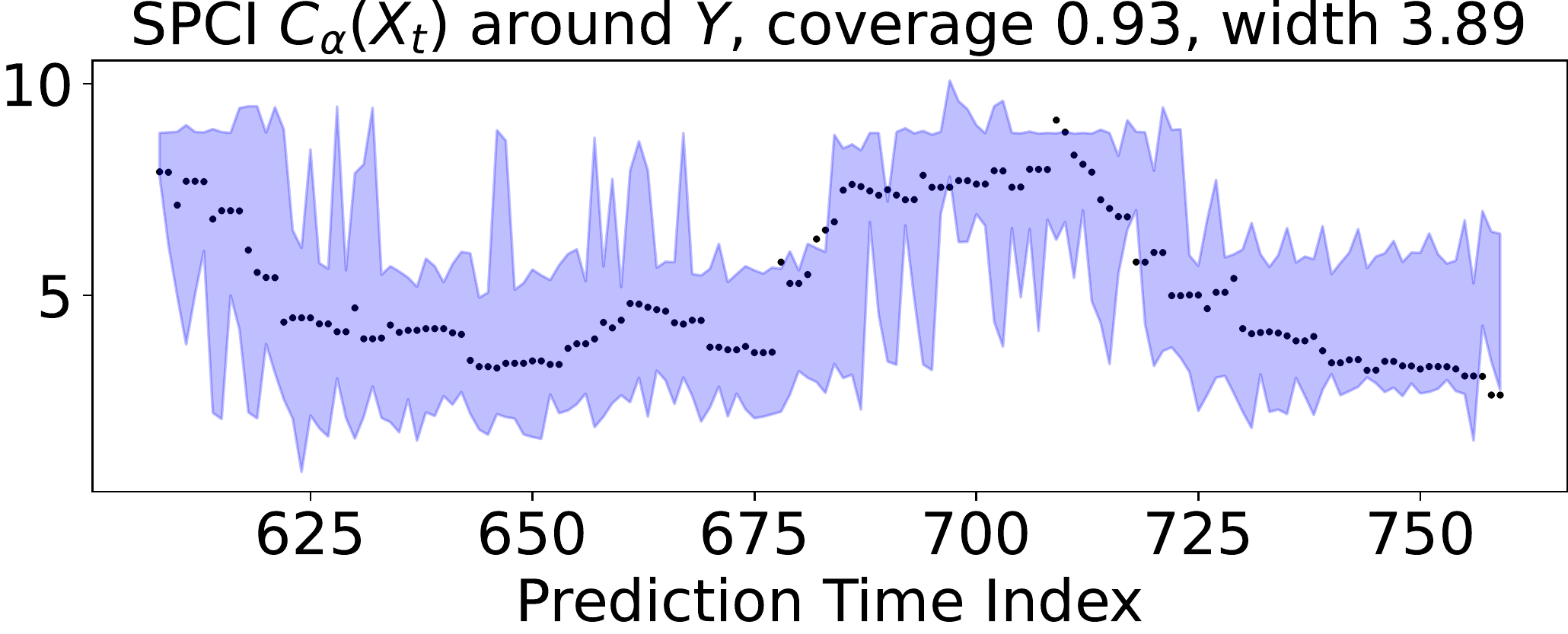}
        \subcaption{SPCI, 4 step ahead}
    \end{minipage}
   \cprotect \caption{Multi-step ahead prediction interval construction by \SPCI{} and \EnbPI{} on wind speed data. Compared to \EnbPI{} results in subfigures (a) and (b), \SPCI{} intervals are much narrower and more adaptive---\SPCI{} intervals follow the trajectory of the time-series whereas \EnbPI{} ones are overly conservative. In addition, \SPCI{} interval increase in width as the predictive horizon increases, reflecting the existence of more uncertainty in long horizons.}
    \label{multi_step_experiment}
    \vspace{-0.25in}
\end{figure}
\vspace{0.1in}
\noindent \textit{(3) Multi-step predictive inference.} In practice, it is often desirable and important to construct $S>1$ prediction intervals at once. This is a challenging problem for \SPCI{} since it involves estimating the conditional \textit{joint} distribution of $S$ residuals ahead. We thus modify \SPCI{} to tackle this problem through a ``divide-and-conquer`` approach. Specifically, we apply \SPCI{} $S$ times on lagged training data $(X_t,Y_{t+s}), s=0,\ldots,S-1$, so that we obtain $S$ fitted QRF estimators to compute the $S$ prediction intervals simultaneously. Additional details including the motivation and algorithm appear in Appendix \ref{append_multi_step}.

Figure \ref{multi_step_experiment} compares \SPCI{} with \EnbPI{} on the wind dataset in terms of multi-step ahead coverage and width. We compare with \EnbPI{} because it supports multi-step ahead prediction in the algorithm, although each batch of $S-$step ahead intervals have the same width by construction. We first note that \EnbPI{} intervals are too wide and non-adaptive, as 4-step ahead intervals may even be narrower than 1-step ahead ones. In contrast, \SPCI{} intervals closely follow the trajectory of actual data and are more adaptive: $S-$step ahead intervals with larger $S$ yield wider intervals on average. This increase in width is expected because there are greater uncertainty when predicting more prediction intervals simultaneously.

\vspace{-0.1in}
\section{Conclusions}

In this work, we propose \SPCI, a general framework for constructing prediction intervals for time series. Similar to existing conformal prediction methods, \SPCI{} is model-free and distribution-free, making it applicable to any time series with arbitrary predictive models. Unlike existing CP methods, \SPCI{} fits quantile regression models on \textit{residuals} to utilize temporal dependency among residuals to achieve more adaptive confidence intervals and better coverage. Theoretical analyses verify the asymptotic valid conditional coverage by \SPCI. Experimental results consistently show improved performance by \SPCI{} over existing sequential CP methods. 

In the future, we aim to extend \SPCI{} for constructing confidence regions for multi-variate time-series, by further exploiting the dependency among individual uni-variate time-series and designing non-conformity scores that enable efficient interval construction. How to develop the multi-step \SPCI{} in Algorithm \ref{algo_multi_step} to more precisely capture the joint distribution of future residuals is also a promising direction.

\vspace{-0.1in}
\section*{Acknowledgement}
\vspace{-0.05in}
This work is partially supported by an NSF CAREER CCF-1650913, and NSF DMS-2134037, CMMI-2015787, CMMI-2112533, DMS-1938106, and DMS-1830210.


\clearpage
\bibliography{jmlr_to_SPCI}
\bibliographystyle{icml2023}

\newpage
\appendix
\onecolumn
\setcounter{table}{0}
\setcounter{figure}{0}
\renewcommand{\thetable}{A.\arabic{table}}
\renewcommand{\thefigure}{A.\arabic{figure}}

\section{Proof} \label{proof}

\begin{proof}[Proof of Proposition \ref{dist_free}]
    The proof is standard in conformal prediction literature based on an exchangeability argument. By \eqref{traditional_CP}, we know that 
    \[
        \bP(Y_t \in \widehat{C}_{t-1}(X_t))=\bP(\hat{\epsilon}_t \in [q_{\alpha/2}(\{\mathcal{E}[\mathcal{I}_2]), q_{1-\alpha/2}(\mathcal{E}[\mathcal{I}_2])]).
    \]
    By exchangeability of the original data and the fact that $\hat{f}$ is trained on $(X_t,Y_t),t\in \mathcal{I}_1$, we have $\mathcal{E}[\mathcal{I}_2]=\{\hat{\epsilon}_j\}_{j\in \mathcal{I}_2}$ and $\hat{\epsilon}_t$ are exchangeable. For $p\in [0,1]$, let $q_p:=q_{\alpha/2}(\{\hat{\epsilon}_j\}_{j\in \mathcal{I}_2})$. Thus, by exchangeability, we have
    \begin{align*}
    & \bP(\hat{\epsilon}_t \in [q_{\alpha/2}, q_{1-\alpha/2}]) \\
    = &  \frac{1}{|\mathcal{I}_2|} \sum_{j\in \mathcal{I}_2}\bP(\hat{\epsilon}_j \in [q_{\alpha/2}, q_{1-\alpha/2}]) \\
    = & \frac{1}{|\mathcal{I}_2|} \mathbb{E}\left[\sum_{j\in \mathcal{I}_2} \boldone(\hat{\epsilon}_j \in [q_{\alpha/2}, q_{1-\alpha/2}]) \right]
    =1-\alpha,
    \end{align*}
    where the last equality holds by the definition of the interval $[q_{\alpha/2}, q_{1-\alpha/2}]$.
\end{proof}

\begin{proof}[Proof of Lemma \ref{lem1}]
    First, by \citep[Theorem 1, p.127-128]{RidlerRowe1968AGC}, we know that \eqref{ptwise_converge} implies 
    \begin{equation}\label{unif_converge}
        \sup_{z \in \R} |\hat F(z|x) - F(z|x) | \rightarrow 0 \text{ in probability.}
    \end{equation}
    Recall that $Q(p)$ is unique. Thus, for any $x$, there exists $\epsilon=\epsilon(x)>0$ such that 
    \[
    \delta=\delta(\epsilon):=\min \{p-F(Q(p)-\epsilon|x), F(Q(p)+\epsilon|x)-p\} > 0.
    \]
    Namely, there exists a small pertubation of $Q(p)$ whereby the change in the value of the distribution function is at least positive. Thus, we have that 
    \begin{align*}
        \mathbb P(|\widehat Q(p)-Q(p)|>\epsilon) 
        &\overset{(i)}{=} \mathbb P(|F(\widehat Q(p)|x)-p|>\delta) \\
        &=\mathbb P(|F(\widehat Q(p)|x)- \hat F(\widehat Q(p)|x)|>\delta) \\
        & \leq \mathbb P(\sup_{z\in \R} |F(z|x)-\hat F(z|x)|> \delta).
    \end{align*}
    Note that (i) holds because the event $|\widehat Q(p)-Q(p)|>\epsilon$ means that $\widehat Q(p)$ is at least $\epsilon$ far away from $Q(p)$. By monotonicity of the distribution function $F$, this event implies the occurrence of the event $|F(\widehat Q(p)|x)-p|>\delta$. 

    Now, \eqref{unif_converge} implies the convergence of estimated quantile values, hence finishing the proof.
\end{proof}

To prove Proposition \ref{ptwise_converge_QRF}, we need several assumptions followed by interpretation and examples.
\begin{assumption}\label{Assume0}
Define $U_t:=F(\tilde Y_t|X=\tilde X_t)$ as the quantile of observations $\tilde Y_t$ conditioning on the observed feature $\tilde X_t$, where $U_t\sim \text{Unif}[0,1]$. For a $x\in \mathcal{B}:=Supp(\{\tilde X_t\}_{t\geq 1})$, define the scalar $z[x]:=F(z|X=x)$. Given
\[
g(i,j,x_1,x_2):=\text{Cov}(\boldone(U_i\leq z[x_1]), \boldone(U_j\leq z[x_2])),
\]
we require that for any pair of $x_1,x_2 \in \mathbb B$, 
\begin{align}
   &  g(i,j,x_1,x_2)=g(|i-j|,x_1,x_2) \ \text{for} \ i\neq j. \label{cov}
\end{align}
In addition, \revold{there exists $\tilde g$ such that}
\begin{align} 
    & \revold{g(k,x_1,x_2) \leq \tilde g(k) \ \forall x_1,x_2 \in \mathbb B, k\geq 1} \label{bound_g} \\
    & \lim_{\tilde T\rightarrow \infty}  \left[\int_1^{\tilde T} \int_1^x \tilde g(u)du dx\right]/\tilde T^2 \rightarrow 0. \label{g_growth}
\end{align}
\end{assumption}

In other words, \eqref{cov} assumes that the covariance of the indicator random variables only depends on the difference in index, where this assumption appears widely in the \textit{weak or wide-sense stationary} processes. The difference is that we do not require constant mean values of the indicator variables. In fact, constant mean is impossible, as $\mathbb E[\boldone(U_t\leq z[x])]=z[x]$, whose value changes depending on the conditioning value $x$.
Meanwhile, there is a function $\tilde g(k)$ in \eqref{bound_g} bounding the covariance uniformly over pairs of values $x_1, x_2$, and \eqref{g_growth} further assumes a restriction on the order of growth of the function $\tilde g(k)$.
Below are examples of $\tilde g(k)$ for which \eqref{g_growth} holds and we can also characterize the decay rate of \eqref{g_growth}.
\begin{example}[Finite memory]\label{ex_finite}
For some cutoff index $s \in \mathbb Z$ and constants $\{c_1,\ldots,c_s\}$,
    $$\tilde g(k)=\begin{cases}
         c_k & k \leq s\\
         0 & k>s
    \end{cases}$$
\end{example}
Showing $\tilde g(k)$ in Example \ref{ex_finite} satisfies \eqref{g_growth} is trivial, with decay rate $O(1/\tilde T^2)$. This example appears in stochastic processes with finite memory.

\begin{example}[Linear decay]\label{ex_linear}
    For every $k\geq 1$,
    $\tilde g(k)=\frac{1}{k^p}, p\geq 1$.
\end{example}
Example \ref{ex_linear} is weaker than Example \ref{ex_finite}. To characterize the decay rate, we see that
\begin{align*}
    \int_1^{\tilde T} \int_1^x \tilde g(u)du dx
    &\leq \int_1^{\tilde T} \int_1^x 1/u du dx \\ 
    &=\int_1^{\tilde T} \log(x) dx = \tilde T(\log \tilde T -1).
\end{align*}
Thus, $\tilde T^{-2}\int_1^{\tilde T} \int_1^x \tilde g(u)du dx\leq\frac{\tilde T(\log \tilde T -1)}{\tilde T^2}=O(\log(\tilde T)/\tilde T)$. Hence, \eqref{g_growth} is proven for Example \ref{ex_linear}.
\begin{example}[Logarithmic decay]\label{ex_log}
For every $k\geq 1$,
    $\tilde g(k) = \left[ \frac{1}{\log(k+1)}\right]^p, p \geq 1$.
\end{example}
Example \ref{ex_log} is weaker than the above two examples as it imposes a weaker decay order on the covariance. Lemma \ref{small_lemma} presents the proof of \eqref{g_growth} for this example, which decays at the order of $O(\frac{1}{2\log \tilde T})$. In general, we wish to show \eqref{g_growth} in this example when $p\in (0,1)$. However, doing so is difficult as the analysis of the integral $\int_1^{\tilde T} \int_1^x [\frac{1}{\log(u+1)}]^p dudx$ is complicated. Furthermore, note that $\log(u+1)^p \rightarrow 1$ as $p\rightarrow 0$, so this integral tends to $\tilde T^2/2$, whereby \eqref{g_growth} cannot be obtained for small enough $p$.
\begin{lemma}\label{small_lemma}
For $p\geq 1$, we have
\[\lim_{\tilde T\rightarrow \infty} \left[\int_1^{\tilde T} \int_2^x \frac{1}{\log(u)^p} dudx\right]/\tilde T^2 = O\left(\frac{1}{2\log \tilde T}\right).\]

\end{lemma}
\begin{proof}[Proof of Lemma \ref{small_lemma}]
First, consider the case where $p=1$. Define $li(x)$ as the anti-derivative of $1/\log(x)$. 
To find the growth order of $li(x)$, we note that $li(x)=Ei(\log x)$, where $Ei(x)$ standards for the \textit{exponential integral} with the form $Ei(x)=\int_{-\infty}^x \frac{e^t}{t}dt$. This can be shown via the change of variable $\log(u)=t$. Note that we have the following asymptotic expansion for $Ei(x)$ \citep{Cody1969ChebyshevAF}:
    \begin{align*}
    Ei(x)
    &=\frac{\exp(x)}{x}(1+\frac{1}{x}+\frac{2}{x^2}+\frac{6}{x^3}+\ldots) \\
    &=\frac{\exp(x)}{x}(1+O(1/x)) \text{ when } x>1.
    \end{align*}
    Thus, $Ei(\log x)=\frac{x}{\log x}(1+O(1/\log x)) \approx \frac{x}{\log x}$ for large $x$.

    As a result, dropping the constants and small order terms yield
    \begin{align*}
        \int_1^{\tilde T} \int_2^x \frac{1}{\log(u)} dudx
        & = \int_1^{\tilde T} Ei(\log x)dx\\
        & = \int_1^{\tilde T} \frac{x}{\log x}dx \\
        & = Ei(2\log \tilde T)
    \end{align*}
    Hence, we have 
    \begin{align*}
        \lim_{\tilde T\rightarrow \infty} \left[\int_1^{\tilde T} \int_2^x \frac{1}{\log(u)} dudx\right ] / \tilde T^2
        & = \lim_{\tilde T\rightarrow \infty} Ei(2\log \tilde T)/\tilde T^2\\
        & = O(\frac{1}{2 \log \tilde T}).
    \end{align*}

Lastly, when $p>1$, $\frac{1}{\log u} > [\frac{1}{\log u}]^p$ uniformly for all $u>1$. Hence, we have 
\[
\lim_{\tilde T\rightarrow \infty} \left[\int_1^{\tilde T} \int_2^x \frac{1}{\log(u)^p} dudx\right ] / \tilde T^2 < \lim_{\tilde T\rightarrow \infty} \left[\int_1^{\tilde T} \int_2^x \frac{1}{\log(u)} dudx\right ] / \tilde T^2,
\]
where the latter limit decays at order $O(\frac{1}{2\log \tilde T})$ as shown above.
\end{proof}

\begin{assumption}\label{Assume1}
The weights $w_t(x)$ in \eqref{w_t_forest} satisfies that for all $x \in \mathbb B$, $w_t(x)=O(1/\tilde T)$.
\end{assumption}
Assumption \ref{Assume1} imposes the condition on the decay order of each weights. Note that by the definition of $w_t(x)$ in \eqref{w_t_forest} and \citep[Assumption 2]{Meinshausen2006QuantileRF}, we know that $w_t(x)=o(1)$. 
Assumption \ref{Assume1} thus assumes an exact order of decay of the weights.
\begin{assumption}\label{Assume2}
The true conditional distribution function is Lipschitz continuous with parameter $L$. That is, for all $x,x'$ in the support of the random variable $X$.
\[
\sup_z |F(z|X=x) - F(z|X=x')|\leq L \|x-x'\|_1.
\]
\end{assumption}
\begin{assumption}\label{Assume3}
For every $x$ in the support of $X$, the conditional distribution function $F(z|X=x)$ is continuous and strictly monotonically increasing in $z$.
\end{assumption}
We remark that Assumption \ref{Assume2} and \ref{Assume3} are identical to \citep[Assumption 4 and 5]{Meinshausen2006QuantileRF}, respectively.

\begin{proof}[Proof of Proposition \ref{ptwise_converge_QRF}]

The proof is motivated by the analyses in \citep{Meinshausen2006QuantileRF}, which assumes $(\tilde{Y}_t, \tilde{X}_t), t\geq 1$ are independent and identically distributed. In essence, we analyze the point-wise difference between the estimate $\hat F(z|x)$ in \eqref{QRF} and the true value $F(z|x)$. The difference can then be broken into two terms. Both terms can be bounded by Chebyshev inequalities, leading to convergence to zero.

For each observation $t=1,\ldots,\tilde T$, denote $U_t:= F(\tilde Y_t|X=\tilde X_t)$ as the quantile of the $t$-th empirical residual $\tilde{Y}_t$. Note that $U_t \sim \text{Unif}[0,1]$ by the property of the distribution function, which is continuous by Assumption \ref{Assume3}.

By the form of the estimator $\hat F(z|x)$ in \eqref{QRF}, we break it into two parts:
\begin{align*}
    \hat F(z|x)
     =& \sum_{t=1}^{\tilde T} w_t(x) \boldone(\tilde Y_t\leq z)\\
     \overset{(i)}{=} &\sum_{t=1}^{\tilde T} w_t(x) \boldone(U_t \leq F(z|\tilde X_t))\\
      =&
     \sum_{t=1}^{\tilde T} w_t(x) \boldone(U_t \leq F(z|x))+  \sum_{t=1}^{\tilde T} w_t(x) (\boldone(U_t \leq F(z|\tilde X_t))-\boldone(U_t \leq F(z|x))).
\end{align*}
The equivalence (i) holds because the event $\{\tilde Y_t \leq z\}$ is identical to the event $\{U_t \leq F(z|X=\tilde X_t)\}$ under Assumption \ref{Assume3}. Thus, we have that
\begin{align*}
    |\hat F(z|x)-F(z|x)|
     \leq & \underbrace{\left|\sum_{t=1}^{\tilde T} w_t(x) \boldone(U_t \leq F(z|x))-F(z|x)\right|}_{(a)} + \\
     & \underbrace{\left|\sum_{t=1}^{\tilde T} w_t(x) (\boldone(U_t \leq F(z|\tilde X_t))-\boldone(U_t \leq F(z|x)))\right|}_{(b)}.
\end{align*}

\noindent \textit{1) Bound of term (a).} The first term can be bounded using Chebyshev inequality. Let $z':=F(z|x)$. Define $U':=\sum_{t=1}^{\tilde T} w_t(x) \boldone(U_t \leq z')$. By the linearity of expectation taken over $U_t$, we have
\begin{align*}
    \mathbb E[U'] 
    &= \sum_{t=1}^{\tilde T} w_t(x) \mathbb E[ \boldone(U_t \leq z')] \\
    &= \left[\sum_{t=1}^{\tilde T} w_t(x)\right] z' \overset{(i)}{=} z', 
\end{align*}
where (i) holds under the definition of $w_t(x)$ in \eqref{w_t_forest}, which satisfies $\sum_{t=1}^{\tilde T} w_t(x)=1$ as remarked earlier. Now, for any $\epsilon >0$,
\begin{align*}
    & \mathbb P\left(\left|\sum_{t=1}^{\tilde T} w_t(x) \boldone(U_t \leq F(z|x))-F(z|x)\right|\geq \epsilon\right) \\
     = &\mathbb P(|U'-z'|\geq \epsilon)  \leq \text{Var}(U')/\epsilon^2.
\end{align*}
Note that 
\begin{align}
\text{Var}(U')
=&\text{Var}(\sum_{t=1}^{\tilde T} w_t(x) \boldone(U_t \leq z')) \nonumber \\ 
=&
\underbrace{\sum_{t=1}^{\tilde T} w_t(x)^2 \text{Var}(\boldone(U_t\leq z'))}_{(i)} + \underbrace{\sum_{i\neq j} w_i(x)w_j(x) \text{Cov}(\boldone(U_i\leq z'),\boldone(U_j\leq z'))}_{(ii)}. \label{eq1}  
\end{align}

We need to show that (i) and (ii) in \eqref{eq1} both converge to zero. To show the convergence of (i), we have $w_t(x)=O(1/\tilde T)$ by Assumption \ref{Assume1} and note that $\text{Var}(\boldone(U_t\leq z'))=\mathbb E(\boldone(U_t\leq z')^2)-E(\boldone(U_t\leq z'))^2=z'-z'^2$. Hence, $\text{Var}(\boldone(U_t\leq z'))<1$ and we have $\sum_{t=1}^{\tilde T} w_t(x)^2 \text{Var}(\boldone(U_t\leq z'))<\sum_{t=1}^{\tilde T} w_t(x)^2=O(1/\tilde T).$

To show the convergence of (ii), we have by Assumption \ref{Assume0} that 
\begin{align*}
    \sum_{i\neq j} w_i(x)w_j(x) \text{Cov}(\boldone(U_i\leq z'),\boldone(U_j\leq z')) 
    \leq & \sum_{k=1}^{\tilde T-1} O\left(\frac{\tilde T-k}{\tilde T^2}\right) \tilde g(k) \\
    \leq & \int_1^{\tilde T} O\left(\frac{\tilde T-k}{\tilde T^2}\right) \tilde g(k) dk \\ 
    = & O\left(\tilde T^{-1}\right) \int_1^{\tilde T} \tilde g(k)dk - O\left(\tilde T^{-2}\right) \int_1^{\tilde T} k\tilde g(k)dk \\ 
    = &O\left(T^{-1}\right) [G(\tilde T)-G(1)]- O\left(\tilde T^{-2}\right) \int_1^{\tilde T} k\tilde g(k)dk,
\end{align*}
where $G(x):=\int_1^x \tilde g(k)dk$ is the anti-derivative. Using integration by part with $u=k, dv=\tilde g(k)dk$, we have
\[
\int_1^{\tilde T} k \tilde g(k)dk = \tilde TG(\tilde T)-G(1) - \int_1^{\tilde T} G(x)dx.
\]
Thus, dropping constants and small order terms yield 
\[
\sum_{i\neq j} w_i(x)w_j(x) \text{Cov}(\boldone(U_i\leq z'),\boldone(U_j\leq z'))\leq \left[\int_1^{\tilde T}\left[\int_1^x \tilde g(k)dk\right]dx\right]/\tilde T^2.
\]
By \eqref{g_growth} in Assumption \ref{Assume0}, we thus have the desired convergence result.

\vspace{0.1in}

\noindent \textit{2) Bound of term (b).} Define $W:=\sum_{t=1}^{\tilde T} w_t(x) \boldone(U_t \leq F(z|\tilde X_t))$. Note that $
\mathbb E(W) = \sum_{t=1}^{\tilde T} w_t(x) F(z|\tilde X_t)$. We have for any $\epsilon>0$,
\begin{align*}
    & \ \mathbb P(|W-\mathbb E(W)|>\epsilon) \\
\leq & \ \text{Var}(W)/\epsilon^2\\
= & \ (\epsilon)^{-2}\left[\sum_{t=1}^{\tilde T} w_t(x)^2 \text{Var}(\boldone(U_t\leq F(z|\tilde X_t))) + \sum_{i\neq j} w_i(x)w_j(x) \text{Cov}(\boldone(U_i\leq F(z|\tilde X_i)),\boldone(U_j\leq F(z|\tilde X_j)))\right].
\end{align*}
By the same argument for bounding term (a) above, we have that $W \overset{p}{\rightarrow} \mathbb E[W]$ as sample size $\tilde T \rightarrow \infty$.

As a result, we have
\[
\left|\sum_{t=1}^{\tilde T} w_t(x) (\boldone(U_t \leq F(z|\tilde X_t))-\boldone(U_t \leq F(z|x)))\right| \overset{p}{\rightarrow} \left|\sum_{t=1}^{\tilde T} w_t(x)(F(z|\tilde X_t)-F(z|x))\right|.
\]
By Assumption \ref{Assume2}, we have
\[
\left|\sum_{t=1}^{\tilde T} w_t(x)(F(z|\tilde X_t)-F(z|x))\right|\leq \sum_{t=1}^{\tilde T} w_t(x) L\|\tilde X_t-x\|_1.
\]
The rest of proof follows due to \citep[Lemma 2]{Meinshausen2006QuantileRF}, which shows that 
\[\sum_{t=1}^{\tilde T} w_t(x) \|\tilde X_t-x\|_1=o_p(1).\] 
\end{proof}

\begin{proof}[Proof of Theorem \ref{thm_non_exchange}]

Under \SPCI{} interval construction in \eqref{SPCI_interval}, the equivalence in \eqref{esti_prob} implies that
\[
\bP(Y_t \in \widehat{C}_{t-1}(X_t)|X_t) = F(\hatQtz{1-\alpha+\hat{\beta}}|\mathcal{E}_t^w)-F(\hatQtz{\hat{\beta}}|\mathcal{E}_t^w),
\]
where $\widehat Q_t(p), p \in [0,1]$ is the estimated $p$-th quantile of $\hat{\epsilon}_t$, $F(z|\mathcal{E}_t^w)$ is the unknown distribution function of $\hat{\epsilon}_t$, and $\betaReal$ minimizes interval width per the procedure in Algorithm \ref{algo_SPCI}.

To finish the proof, by Proposition \ref{ptwise_converge_QRF}, we know that the conditional distribution estimator $\hat F(z|\mathcal{E}_t^w)$ using QRF converges point-wise to the true $F(z|\mathcal{E}_t^w)$ as the sample size (hence the number of residuals) approaches infinity. By Lemma \ref{lem1}, we thus know that $\widehat Q_t(p) \rightarrow Q_t(p)$ in probability for all $p \in [0,1]$. 

We can thus use the continuous mapping theorem \citep[Theorem 2.3]{van2000asymptotic} to finish the proof: by Assumption \ref{Assume2}, the true conditional distribution function $F$ is absolutely continuous and therefore differentiable almost everywhere. Thus, the set of discontinuity points of $F$ has measure zero. As the number of data $\tilde T \rightarrow \infty$ when training QRF, we finally have that in probability, 
\begin{align*}
    &F(\hatQtz{1-\alpha+\hat{\beta}}|\mathcal{E}_t^w)-F(\hatQtz{\hat{\beta}}|\mathcal{E}_t^w)\\
     \rightarrow & F(Q_t(1-\alpha+\hat{\beta})|\mathcal{E}_t^w)-F(Q_t(\hat{\beta)}|\mathcal{E}_t^w)
    = 1-\alpha. 
\end{align*}
\end{proof}

\section{Experimental details}\label{append_exp}

\noindent \textit{(1) Baseline methods.} We compare \SPCI{} with three recent CP methods for non-exchangeable data or time series, which have also been carefully described in the literature review. In particular, they all leverage the feedback $Y_t$ after it is sequentially revealed.
\begin{itemize}
    \item \EnbPI{} \citep{xu2021conformal} proposes a general framework for constructing time-series prediction intervals. In particular, it fits LOO regression models and uses residuals as non-conformity scores. Comparing our use of \SPCI{} in experiments, the only difference appears in using conditional rather than empirical quantiles for the calibration of interval width.
    \item AdaptiveCI \citep{Gibbs2021AdaptiveCI} is an adaptive procedure that adjusts the significance level $\alpha$ based on historical information of interval coverage. It leverages CQR \citep{CPquantile} to produce intervals that maintain coverage validity in theory. We use the quantile random forest as the predictor and update $\alpha$ according to the simple online update (ibid., Eq (2)).
    \item NEX-CP \citep{RinaNonExchange} uses weighted quantiles to tackle arbitrary distribution drift in test data. In particular, the implementation is based on full conformal with weighted least squares regression models, which empirically yields more stable coverage than the naive split conformal method.
\end{itemize}

\vspace{0.1in}
\noindent \textit{(2) Real-data description.} We describe the three real time-series for results in Section \ref{sec_real_data}. The first dataset is the wind speed data (m/s) at wind farms operated by the Midcontinent Independent System Operator (MISO) in the US \citep{Zhu2021MultiResolutionSP}. The wind speed record was updated every 15 minutes over a one-week period in September 2020. The second dataset contains solar radiation information\footnote{Collected from National Solar Radiation Database (NSRDB): https://nsrdb.nrel.gov/.} in Atlanta downtown, which is measured in Diffuse Horizontal Irradiance (DHI). The full dataset contains a yearly record in 2018 and is updated every 30 minutes. We remark that uncertainty quantification for both wind and solar is important for accurate and reliable energy dispatch. The last dataset tracks electricity usage and pricing \citep{Harries1999SPLICE2CE} in the states of New South Wales and Victoria in Australia, with an update frequency of 30 minutes over a 2.5-year period in 1996–1999. We are interested in tracking the quantity of electricity transferred between the two states.

\subsection{Simulation}\label{append_simul}

We first describe details regarding data simulation procedures. We then show additional rolling coverage and width results when comparing with NEX-CP.

\subsubsection{Data simulation}
For the results in Table \ref{simul_results}, we simulate the non-stationary and heteroskedastic time-series as follows:
\begin{enumerate}
    \itemsep 0em
    \item Non-stationary (Nstat) time-series: We let 
    \begin{align}
        & f(X_t)=g(t)h(X_t). \label{non_stationary} \\
        & g(t)=\log(t')\sin(2\pi t'/12), t' = \text{mod}(t,12).\nonumber \\
        & h(X_t) = (|\beta^TX_t|+(\beta^TX_t)^2+|\beta^TX_t|^3)^{1/4}.\nonumber
    \end{align}
    Note that the model in \eqref{non_stationary} can represent non-stationary time-series due to additional time-related effects (e.g., time drift, seasonality, periodicity, etc.). For a fixed window size $w\geq 1$, each feature observation $X_t=[Y_{t-w},\ldots,Y_{t-1}]$ contains the past $w$ observations of the response $Y$. We sample the errors $\epsilon_t$ from an AR(1) process, where $\epsilon_t=\rho \epsilon_{t-1}+e_t$ and $e_t$ are {\it i.i.d.} normal random variables with zero mean and unit variance with $\rho=0.6$. 

    \revold{We want to compare the performance of \EnbPI{} and \SPCI{} assuming no feature mis-specification, so that the only difference in interval coverage/width lies in how the residuals are used to construct the intervals. Therefore, because $f$ in \eqref{non_stationary} explicitly depends on $t$ and $X_t$, we use the new feature $\tilde{X}_t := [\text{mod}(t,12), X_t]$ to predict $Y_t$. We acknowledge that in practice, the true periodicity constant 12 in \eqref{non_stationary} is unknown, and one must estimate it before constructing the new feature $\tilde{X}_t$. Meanwhile, Table \ref{simul_mod_unknown} compares all four CP methods when the time information is unknown (i.e., $\tilde{X}_t=X_t$), and we still observe much narrower intervals by \SPCI{} than the baselines.}

    \begin{table}[!t]
    \caption{Simulation on non-stationary time-series: the setup is identical to Table \ref{simul_results}. We compare \SPCI{} against baseline CP methods when no time information is assumed known (i.e., $\tilde{X}_t=X_t$).}
    \label{simul_mod_unknown}
    \centering
    \resizebox{\linewidth}{!}{
    \begin{tabular}{llllllll}
    \toprule
    \multicolumn{2}{c}{SPCI} & \multicolumn{2}{c}{EnbPI} & \multicolumn{2}{c}{AdaptiveCI} & \multicolumn{2}{c}{NEX-CP} \\
    \midrule
    Coverage & Width & Coverage & Width & Coverage & Width & Coverage & Width \\
    0.92 (2.75e-3)	& 12.96 (2.56e-2) & 0.90 (2.21e-3)	& 25.41 (4.79e-2) & 0.90 (4.12e-3)	& 28.00 (5.81e-2) & 0.93 (3.10e-3)	&  46.50 (6.29e-2)\\
    \bottomrule
    \end{tabular}
    }
    \vspace{-0.2in}
\end{table}

    \item Heteroskedastic (Hetero) time-series: We let
    \begin{align}
         & f(X_t)=(|\beta^T X_t|+(\beta^T X_t)^2+|\beta^T X_t|^3)^{1/4}. \label{hetero} \\
         & \text{Var}(\epsilon_t)=\sigma(X_t)^2, \sigma(X_t)=\textbf{1}^T X_t.
    \end{align}
    Note that the model above represents the generalized autoregressive conditional heteroskedasticity (GARCH) model \citep{GARCH}, where variances of response $Y_t$ depend on its feature $X_t$. We let features $X_t \in \mathbb{R}^{20}$, with i.i.d. entries from $\text{Uniform}[0,e^{0.01\text{mod}(t,100)})$. Due to heteroskedastic errors, we estimate conditional quantile of normalized residuals $\hat{\epsilon}_t:=(Y_t-\hat{f}_t(X_t))/\hat{\sigma}(X_t)$ and multiply the quantile values by estimates $\hat{\sigma}(X_t)$ to construct the prediction intervals.
\end{enumerate}

For the simulated results in Table \ref{simul_results_with_NEXCP}, the data with distribution-shift and change-points are simulated as follows. For $N=2000$ and $X_i \sim \mathcal N(0,\boldsymbol I_4), i=1,\ldots,N$:
\begin{enumerate}
    \item Distribution-drift (Drift): $Y_i\sim X_i^T\beta_i + \mathcal{N}(0,1)$, where $\beta_1 = (2,1,0,0), \beta_N=(0,0,2,1)$, and $\beta_i, i=2,\ldots,N-1$ is a linear interpolation of $\beta_1$ and $\beta_N$.
    \item Changepoints (Change): $Y_i\sim X_i^T\beta_i + \mathcal{N}(0,1)$,
    \begin{align*}
        & \beta_1=\ldots=\beta_{500} = (2,1,0,0) \\
        & \beta_{501}=\ldots=\beta_{1500} = (0,-2,-1,0) \\
        & \beta_{1501}=\ldots=\beta_N = (0,0,2,1).
    \end{align*}
\end{enumerate}
Similar to NEX-CP, we apply \SPCI{} after a burn-in period of the first $100$ sample points, and in addition, adaptively refit the point estimator $\hat{f}$ using a rolling window of $\min(T, T_0)$ points during testing for $T=101,\ldots,2000$. We choose $T_0=300$ under distribution shifts and $T_0=200$ under changepoints. Similar to NEX-CP, we use weighted linear regression with exponentially decaying weights to train the point estimator $\hat{f}$ in \SPCI{}. 

\subsubsection{Comparison with NEX-CP} 

We explain why \SPCI{} tends to under-cover in these settings before $\alpha$ adjustment. We suspect the primary reasons are that prediction residuals $\hat{\epsilon}_i$ in these settings are \textit{(nearly) independent yet non-identically distributed}. More precisely, regarding independence, suppose we use the split conformal framework in \SPCI{} to train $\hat{f}$ and obtain residuals on the calibration set. We thus have that for each prediction residual $\hat{\epsilon}_i$ in the calibration set,
\begin{equation}\label{simul_eps}
\hat{\epsilon}_i = Y_i-\widehat{Y}_i \sim (X_i^T\beta_i + \mathcal{N}(0,1))-\hat{f}(X_i).
\end{equation}
Note that $X_i$ are all independent by design. Except for the possible dependency in $\beta_i$, which is zero in the change-point setting, the (unobserved) test residual $\hat \epsilon_{T+1} \indep \hat \epsilon_{T+1-k}, k \geq 1$, where $\indep$ denotes independence of random variables. We empirically verify the independence of residuals through the PACF plot in Figure \ref{simul_pacf}. On the other hand, regarding non-identical distribution, because of drifts or changepoints through the changes in $\beta_i$, the residuals do not follow the same distribution. Thus, the QRF estimated on past residuals may not be a desirable estimator for the conditional quantile of the test residual $\hat \epsilon_{T+1}$, hence weakening the performance of \SPCI{} in this setting.

\begin{figure}[!t]
    \centering
    \begin{minipage}{0.48\textwidth}
       \includegraphics[width=\linewidth]{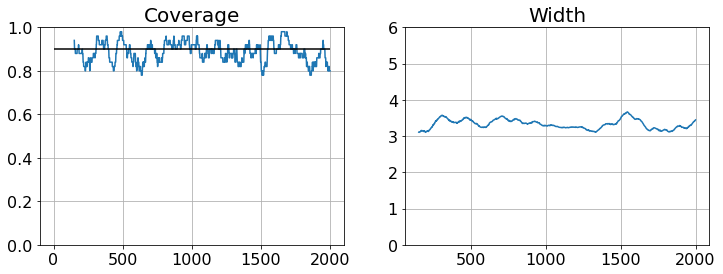}
       \subcaption{Distribution shift}
    \end{minipage}
    \begin{minipage}{0.48\textwidth}
       \includegraphics[width=\linewidth]{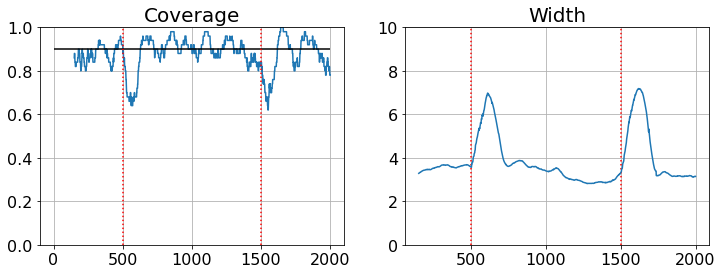}
       \subcaption{Changepoint}
    \end{minipage}
    \caption{\revold{Rolling coverage and width during test time without adjusted $\alpha$ values. Target coverage at 0.9 is marked in the black lines. In (b), the two changepoints are marked in dotted red line at time indices 500 and 1500.}}
    \label{rolling_cov_simul}
\end{figure}

\begin{figure}[!b]
    \centering
    \begin{minipage}{0.48\linewidth}
       \includegraphics[width=\linewidth]{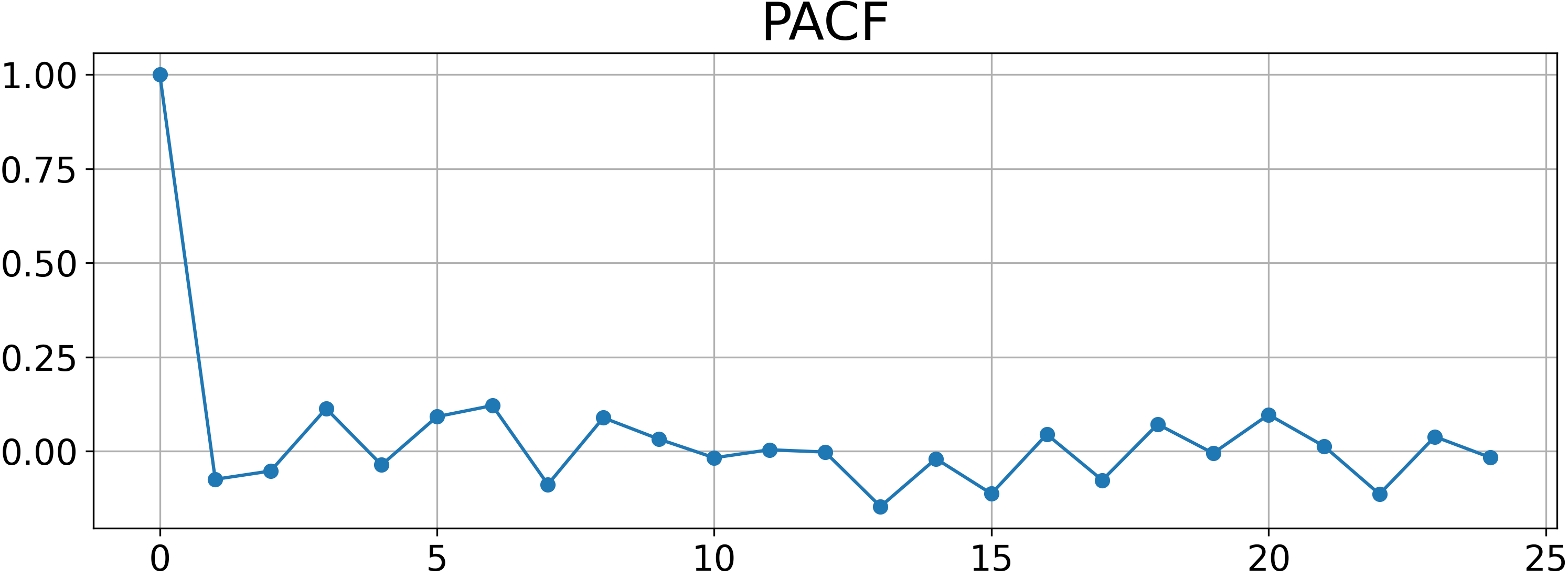}
       \subcaption{Distribution shift}
    \end{minipage}
    \begin{minipage}{0.48\linewidth}
       \includegraphics[width=\linewidth]{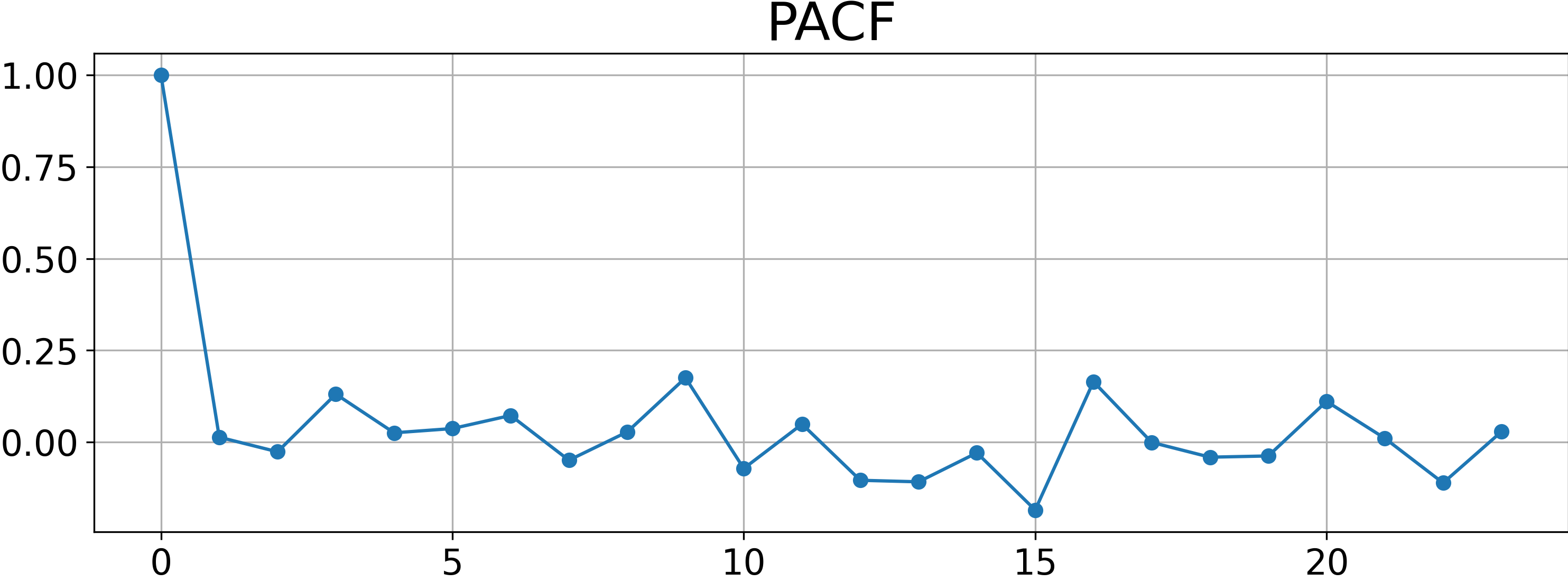}
       \subcaption{Changepoint}
    \end{minipage}
    \caption{\revold{PACF using 300 residuals (dist. shift) and 200 residuals (change-point). We see near independence of the residuals, which are non-identically distributed due to the data generation.}}
    \label{simul_pacf}
\end{figure}

\subsection{Additional real-data comparisons}

We compare \SPCI{} with AdaptiveCI on stock market data. Specifically, the dataset is publicly available on Kaggle \url{https://www.kaggle.com/datasets/paultimothymooney/stock-market-data}, where we are interested in constructing the prediction intervals for the closing price. We randomly select three NASDAQ stock from three companies. Table \ref{stock_results} shows several findings:
\begin{itemize}
    \item When AdaptiveCI and SPCI both yield valid coverage (on company AJISF), the width of SPCI is significantly narrower.
    \item Even when AdaptiveCI loses coverage and SPCI maintains coverage (on company AGTC), the width of SPCI is still significantly narrower.
    \item When both methods lose coverage (on company AAVL), the loss by SPCI is less and SPCI still yields narrower intervals.
\end{itemize}

\begin{table}[!t]
    \vspace{-0.15in}
    \caption{Marginal coverage and width by \texttt{SPCI} and AdaptiveCI on three NASDAQ stock market data. The tarrget coverage is 0.9, and entries in the bracket indicate standard deviation over three independent trials.}
    \label{stock_results}
    \centering
    \resizebox{\linewidth}{!}{
    \begin{tabular}{ccccccc}
    \toprule
    Method & \multicolumn{2}{c}{Company AJISF} & \multicolumn{2}{c}{Company AGTC} & \multicolumn{2}{c}{Company AAVL}\\
    \midrule
    & Coverage & Width & Coverage & Width & Coverage & Width \\
    SPCI & 0.89 (2.34e-3) &	17.64 (1.24e-1) &  0.95 (2.43e-3)	& 2.89 (5.23e-2) & 0.81 (3.64e-3)	& 1.03 (2.34e-2)\\
    AdaptiveCI & 0.94 (3.43e-3)	& 30.20 (2.53e-1)	& 0.71 (1.53e-2)	& 5.88 (7.43e-2)	& 0.64 (2.32e-2)	& 2.18 (3.37e-2)\\
    \bottomrule
    \end{tabular}
    }
\end{table}

\subsection{Multi-step inference}\label{append_multi_step}

\noindent \textit{(1) Motivation and setup.} We first motivate the study of multi-step ahead prediction interval. For examples in Section \ref{sec_real_data}, all intervals are one step ahead: the response variable $Y_t$ is revealed \textit{before} $\Ctalpha$ is constructed, which is the prediction interval for $Y_{t+1}$. Such immediate feedback is advantageous for all adaptive methods as they thus have access to the most up-to-date information about the data process. Nevertheless, such access can be neither feasible nor desirable for some use cases. In energy systems such as wind or solar prediction, we often need multiple forecasts spanning a long enough future horizon to allow enough time for subsequent dispatch. Meanwhile, lags in data collection can limit the availability of feedback---for $S>1$, $Y_t$ may not be revealed until all $S$ intervals ahead are constructed.

We consider the following multi-step ahead prediction setting. Fix a value of $S \geq 1$, which denotes the $s-$step ahead prediction setting ($S=1$ refers to examples in earlier sections). Features $X_t=[Y_{t-1},\ldots,Y_{t-\tau}]$ are auto-regressive with a pre-specified window $\tau \geq 1$. At prediction time $t$, we need to construct $S$ prediction intervals at once for time indices $t,\ldots,t+S-1$. In particular, responses $Y_t,\ldots,Y_{t+S-1}$ (and thus features $X_{t+1},\ldots,X_{t+S}$) are not available until we construct prediction intervals at indices $t+S,\ldots,t+2S-1$. 

\vspace{0.1in}
\noindent \textit{(2) Multi-step \SPCI{} algorithm.} Note that constructing multi-step ahead prediction intervals using \SPCI{} involves estimating the \textit{joint} distribution of $\hat{\epsilon}_{t+1},\ldots,\hat{\epsilon}_{t+S}$ every $S$ test indices. Doing so can be highly challenging. Instead, we take a simplified ``divide-and-conquer'' approach based on the LOO fitting in \EnbPI{}. First, we train $S$ sets of LOO predictors for estimating the value of $\widehat{Y}_{t+j}, j=0,\ldots,S-1$. This is implemented by fitting $B$ bootstrap models on each lagged data $\{(X_t,Y_{t+s})\}_{t=1}^{T-s+1}, s=1,\ldots,S$. Then, we compute residuals only at $t=1+kS: kS\leq T-1$. We do so because on test data, new feature $X_t$ and output $Y_t$ are revealed only in every $S$ step. Lastly, we fit QRF $S$ times using past residuals with lags to obtain $s$ prediction intervals at once.

We briefly compare and contrast Algorithm \ref{algo_SPCI} (\SPCI) and \ref{algo_multi_step} (multi-step ahead \SPCI) when LOO point predictors are trained. Computationally, we need to refit $S-1$ more sets of LOO predictors in multi-step ahead \SPCI{} for point prediction. On the other hand, both algorithms fit the same number of QRF regressors for constructing prediction intervals. In practice, multi-step \SPCI{} is expected to yield wider intervals as $S$ increases because there is greater uncertainty when fitting the baseline regression or QRF on lagged data. A simple example is the $AR(1)$ process where $x_t=ax_{t-1}+\epsilon_t, \epsilon_t \overset{i.i.d.}{\sim} \mathcal{N}(0,1)$. Using the present feature $x_{t-1}$, we have $x_{t+S}=a^{S+1}x_{t-1} + \sum_{i=1}^S a^{i-1}\epsilon_{t+i}$, whereby the error distribution $a^{i-1}\epsilon_{t+i}\sim N(0,\sum_{i=1}^S a^{2(i-1)})$, so width naturally increases.

\section{Additional technical details}\label{append_detail}

We first present the \SPCI{} algorithm for exchangeable data in Algorithm \ref{algo_SPCI_exchangeable}. We then present the \SPCI{} algorithm for multi-step ahead inference in Algorithm \ref{algo_multi_step}.

\begin{algorithm}[!t]
\cprotect\caption{\SPCI{} for exchangeable data (based on split conformal)}
\label{algo_SPCI_exchangeable}
\begin{algorithmic}[1]
\REQUIRE{Training data $\{(X_t, Y_t)\}_{t=1}^T$, significance level $\alpha$.
}
\ENSURE{Prediction intervals $\widehat{C}_{t-1}(X_t), t>T$}
\STATE Randomly split $\{1,\ldots,T\}$ into disjoint index sets $\mathcal{I}_1$ and $\mathcal{I}_2$.
\STATE Train a point predictor $\hat{f}$ with $\{(X_t,Y_t)\}_{t\in \mathcal{I}_1}$.
\STATE Obtain residuals $\hat{\epsilon}_t:=Y_t-\hat{f}(X_t)$ for $t \in \mathcal{I}_2$.
\FOR {$t>T$}
\STATE Return the prediction interval $\Ctalpha$ as in \eqref{traditional_CP}.
\ENDFOR
\end{algorithmic}
\end{algorithm}

\begin{algorithm}[!t]
\cprotect \caption{Multi-step \texttt{SPCI} (based on LOO prediction in \EnbPI{} {\citep{xu2021conformal}})}
\label{algo_multi_step}
\begin{algorithmic}[1]
\REQUIRE{Training data $\{(X_t, Y_t)\}_{t=1}^T$, significance level $\alpha$, number of bootstrap estimators $B$, aggregation function $\phi$, conditional quantile regression algorithm $\mathcal{Q}$, multi-step size $S>1$.
}
\ENSURE{Prediction intervals $\widehat{C}_{t-1}(X_t), t>T$\\}
\FOR[$\triangleright$ $s$-step ahead model fitting]{$s=1,\ldots,S$}
\STATE Sample with replacement $B$ index sets, each of size $T-s+1$: \\ $\{S_b:S_b\subset \{1,\ldots, T-s+1\}\}_{b=1}^B$.
\STATE Train $B$ corresponding bootstrap estimators $\{\hat{f}^b\}_{b=1}^B$ on data $\{(X_t,Y_{t+s-1}): t\in S_b\}$.\\
\COMMENT{$\triangleright$ Leave-one-out aggregation}
\STATE Initialize $\widehat{\boldsymbol \epsilon}=[\ ]$
\FOR {$t=1,1+S,\dots,1+k S$ such that $k S\leq T-1$}
\STATE $\hat{f}^s_{t}(X_t)=\phi(\{ \hat{f}^b(X_t), t \notin S_b\}_{b=1}^B)$
\STATE $\widehat{\boldsymbol \epsilon}.\text{append}(Y_{t+s-1}-\hat{f}^s_{t}(X_t))$
\ENDFOR 
\ENDFOR 
\FOR[$\triangleright$ Interval construction]{$t>T$}
\STATE Compute $s=\text{mod}(t-T,S+1)$ and $t'=t-s$ \\ \COMMENT{$\triangleright$ $t'$ denotes the most recent index where residual $\hat{\epsilon}_{t'}$ and feature $X_{t'+1}$ are available.}
\IF[$\triangleright$ Fit quantile regressors with updated residuals]{$s = 1$}
\STATE Re-fit $S$ quantile estimators $\{\hatQtz{\cdot \ ;s'}\}_{s'=1}^S$ with $\{(\hat\epsilon_j^w,\hat\epsilon_{j+s'-1})\}_{j=t-T+w}^{t-1-(S-1)}$.
\ENDIF
\STATE Compute $\betaReal = {\arg\min}_{\beta\in[0,\alpha]} (\hatQtz{1-\alpha+\beta; s}-\hatQtz{\beta; s})$ using $\hat\epsilon_{t'}^w$.

\STATE $\Ctalpha=[\hat Y_t+w_{\rm{left}}(t),\hat Y_t+w_{\rm{right}}(t)],$ where \\
$\hat Y_t=\phi(\{\hat{f}^s_{j}(X_{t'+1})\}_{j=1}^{T/S}), w_{\rm{left}}(t)=\hatQtz{\betaReal; s},w_{\rm{right}}(t)=\hatQtz{1-\alpha+\betaReal; s}$ 
.
\ENDFOR
\end{algorithmic}
\end{algorithm}

\end{document}